\documentclass[10pt,twocolumn,letterpaper]{article}

\usepackage[pagenumbers]{cvpr} % To force page numbers, e.g. for an arXiv version

\definecolor{cvprblue}{rgb}{0.21,0.49,0.74}
\usepackage[pagebackref,breaklinks,colorlinks,allcolors=cvprblue]{hyperref}
\usepackage{tikz}
\usepackage{algorithm}
\usepackage{algorithmic}
\usepackage[table]{xcolor}
\usepackage{tabularx}  
\usepackage{amsmath} 
\usepackage{amsthm}  
\usepackage{graphicx}
\usepackage{booktabs}  
\usepackage{mathrsfs} 
\usepackage{enumitem}
\usepackage{bm}
\newtheorem{theorem}{Theorem}  
  
\usepackage{amssymb}  
\usepackage{multirow}
\newtheorem{lemma}{Lemma}
\newtheorem{definition}{Definition}
\newtheorem{example}{Example}
\usepackage[mathscr]{eucal}
\usepackage{amssymb}
\usepackage{pifont}
\usepackage{subcaption}

\newcommand{\dashfbox}[1]{%
  \tikz[baseline=(c.base)]{
    \node[draw,dashed,inner sep=3pt](c){$#1$};
  }%
}

\def\paperID{} % *** Enter the Paper ID here
\def\confName{CVPR}
\def\confYear{2026}

\title{Hierarchical Action Learning for Weakly-Supervised Action Segmentation}

\author{Junxian Huang$^{1}$,\quad Ruichu Cai$^{1, *}$, \quad Hao Zhu$^{1}$, \quad Juntao Fang$^{1}$, \quad Boyan Xu$^{1}$ 
\\ \quad Weilin Chen$^{1}$, \quad Zijian Li$^{2, 3}$, \quad Shenghua Gao$^{4}$\\
$^1$Guangdong University of Technology\\
$^2$Carnegie Mellon University\\
$^3$Mohamed bin Zayed University of Artificial Intelligence\\
$^4$University of Hong Kong\\
{\tt\small huangjunxian459@gmail.com,  leizigin@gmail.com}}
\begin{document}
\maketitle
\def\thefootnote{*}\footnotetext{Corresponding authors.}\def\thefootnote{\arabic{footnote}}
\begin{abstract}
Humans perceive actions through key transitions that structure actions across multiple abstraction levels, whereas machines, relying on visual features, tend to over-segment. This highlights the difficulty of enabling hierarchical reasoning in video understanding. Interestingly, we observe that lower-level visual and high-level action latent variables evolve at different rates, with low-level visual variables changing rapidly, while high-level action variables evolve more slowly, making them easier to identify. Building on this insight, we propose the Hierarchical Action Learning (\textbf{HAL}) model for weakly-supervised action segmentation. Our approach introduces a hierarchical causal data generation process, where high-level latent action governs the dynamics of low-level visual features. To model these varying timescales effectively, we introduce deterministic processes to align these latent variables over time. The \textbf{HAL} model employs a hierarchical pyramid transformer to capture both visual features and latent variables, and a sparse transition constraint is applied to enforce the slower dynamics of high-level action variables. This mechanism enhances the identification of these latent variables over time. Under mild assumptions, we prove that these latent action variables are strictly identifiable. Experimental results on several benchmarks show that the \textbf{HAL} model significantly outperforms existing methods for weakly-supervised action segmentation, confirming its practical effectiveness in real-world applications. \footnote{https://github.com/DMIRLAB-Group/HAL}
\end{abstract}    

\section{Introduction}
\label{sec:intro}

\begin{figure*}
    \centering
    \includegraphics[width=0.9\linewidth]{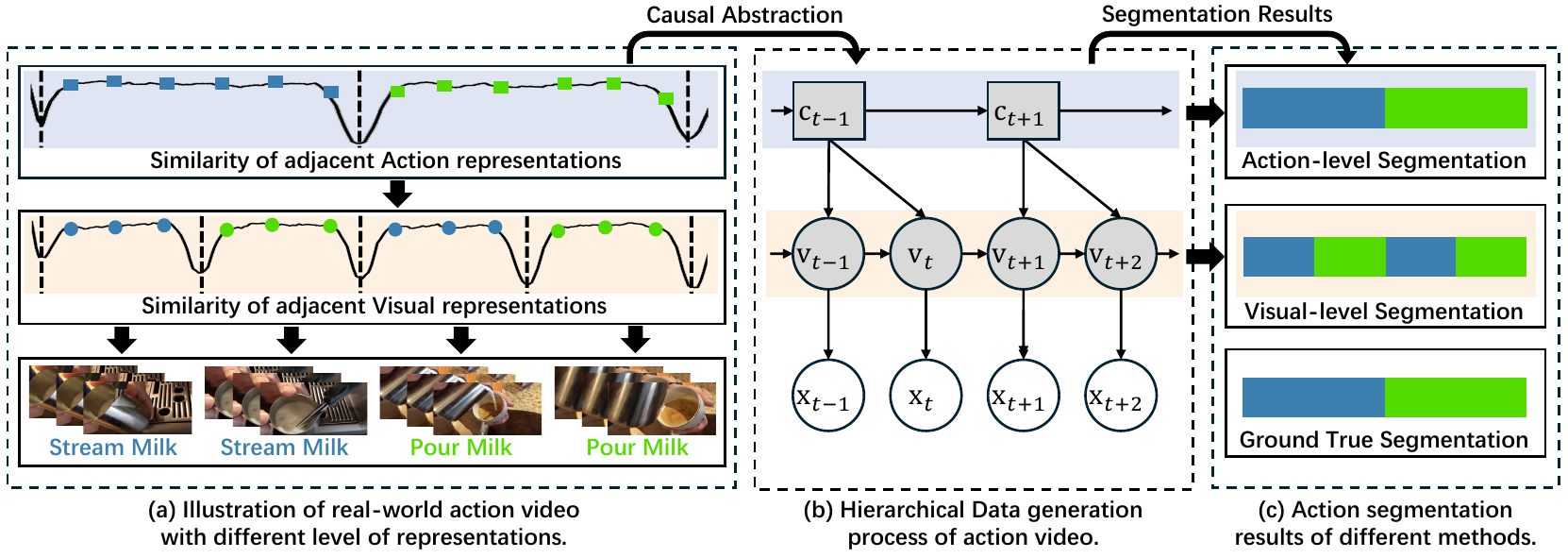}
    \caption{An action segmentation example in the CrossTask dataset. (a) Real-world action videos exhibit different levels of representation, where action representations change more smoothly than visual representations. (b) A hierarchical data generation process of action videos, where high-level latent action variables govern the evolution of low-level visual variables. (c) Action segmentation results show that action-level segmentation better aligns with the ground truth than visual-level segmentation.}
    \label{fig:motivation}
\end{figure*}

Weakly-supervised action segmentation \cite{ding2018weakly}, relying on coarse annotations such as action transcripts \cite{TASL} rather than detailed frame-level labels \cite{lu2024fact, ahn2021refining}, is one of the most fundamental tasks in video understanding \cite{li2024mvbench}, with applications in human activity recognition \cite{hoai2011joint} and video retrieval \cite{sun2021vsrnet}.

Several methods have been proposed for weakly-supervised action segmentation \cite{ding2023temporal, POC, rahaman2022generalized}, among which those based on action transcripts \cite{ding2023temporal} can be divided into iterative two-stage and single-stage approaches. The iterative methods start with initially estimated labels and refine them through repeated alignment and update steps. For example, ISBA \cite{ding2018weakly} first divides the video uniformly and then refines the segmentation using soft labels, while TASL \cite{TASL} progressively aligns training videos with transcripts. Single-stage approaches directly align transcripts with video frames. ATBA \cite{ATBA} identifies key action transitions for alignment, and 2by2 \cite{bueno20242by2} employs a Siamese network to determine whether two videos represent the same activity. Recently, the pose-aware model \cite{zhao2025pose} leverages pose-guided contrastive learning to improve action segmentation in instructional videos. 

Despite the remarkable progress, existing methods mainly rely on visual-level representations, where frequent fluctuations in appearance often lead to over-segmentation and noisy boundaries. As illustrated in Figure~\ref{fig:motivation}(a) with a CrossTask example, visual variations are easily mistaken for action transitions, resulting in false boundaries. In contrast, humans perceive activities through hierarchical action structures, where a few key transitions organize actions across multiple levels of abstraction. Motivated by this observation, we hypothesize that videos contain hierarchical latent variables that evolve at different temporal rates, low-level visual variables change rapidly, while high-level action variables evolve more slowly and capture stable semantic patterns (as shown in Figure~\ref{fig:motivation}(b) and (c)). These high-level action variables summarize the progression of an activity and naturally mitigate over-segmentation.

However, modeling such hierarchical latent structures is non-trivial. Without explicit constraints, the dynamic visual variables can become entangled with the higher-level action representations, producing unstable boundaries. Interestingly, the temporal asymmetry between fast-changing visual features and slowly evolving action variables provides a valuable inductive cue. Enforcing sparsity and smooth transitions on high-level action can therefore disentangle them from visual fluctuations and facilitate identifiability.

Based on this insight, we propose the Hierarchical Action Learning (\textbf{HAL}) model for weakly-supervised action segmentation. We formulate a hierarchical causal data generation process, where high-level latent action governs the dynamics of low-level visual variables. To model their asynchronous evolution, we introduce deterministic temporal alignment processes that maintain consistency between action and visual layers. Moreover, we design a pyramidal transformer architecture that captures multi-level dependencies, coupled with a sparse transition constraint to enforce the slower dynamics of high-level action. This inductive bias enables theoretical identifiability of the action variables over time. Extensive experiments on several benchmarks demonstrate that the proposed HAL model surpasses existing methods, validating its effectiveness in capturing structured, multi-level action dynamics in real-world videos.

\section{Related Works}\label{app:related}

\subsection{Weakly-Supervised Action Segmentation}

Weakly-supervised action segmentation methods \cite{bojanowski2014weakly,connectionistTemporal,kuehne2017weakly,RNN-HMM,richard2017weakly,CDFL,TASL,richard2018neuralnetwork,chang2019d3tw,DPDTW,AdaAct} aim to reduce reliance on frame-level annotations. Four main supervision forms are commonly studied: transcripts, action sets, timestamps, and textual descriptions.
Transcript-based approaches use ordered action lists without frame-level alignment, offering strong annotation efficiency. They include iterative two-stage methods, which generate and refine pseudo labels \cite{kuehne2017weakly,richard2017weakly,RNN-HMM,ding2018weakly,TASL}, and single-stage methods that learn alignment end-to-end \cite{connectionistTemporal,richard2018neuralnetwork,chang2019d3tw,CDFL,Mucon,DPDTW}. While effective, some suffer from high computational cost.
Action-set supervision provides only unordered action collections. Early work \cite{richard2018action} infers likely sequences via contextual and duration models. Later methods \cite{fayyaz2020sct,li2020set,li2021anchor,POC} improve pseudo-label generation and enable end-to-end training through constrained or differentiable formulations.
Timestamp-based methods use sparse frame annotations to infer full segmentations. They often refine pseudo labels iteratively \cite{li2021temporal,rahaman2022generalized,khan2022timestamp,souri2022robust}, employing EM frameworks or graph propagation to improve label consistency and handle missing timestamps.
Text-based supervision provides natural but noisy cues. Prior work aligns text and video \cite{malmaud2015s,bojanowski2015weakly,alayrac2016unsupervised,zhukov2019cross}, or learns joint multimodal embeddings \cite{radford2021learning,shvetsova2022everything}. Recent studies \cite{sener2015unsupervised,fried2020learning,han2022temporal} enhance alignment robustness via generative or structured models.

\subsection{Causal Representation Learning} 
Independent Component Analysis (ICA) is a core approach for uncovering identifiable causal representations \cite{rajendran2024learning,mansouri2023object,wendong2024causal}. Classical ICA assumes that observations arise from a linear mixture of independent latent variables \cite{comon1994independent,hyvarinen2013independent,lee1998independent,zhang2007kernel}, but this assumption rarely holds in real-world nonlinear settings. To guarantee identifiability in nonlinear ICA, additional conditions such as sparsity or auxiliary variables are introduced \cite{zheng2022identifiability,hyvarinen1999nonlinear,khemakhem2020ice,li2023identifying}.
Aapo et al. \cite{khemakhem2020variational} established the first theoretical foundation by showing that, under exponential-family latent distributions, identifiability can be achieved via auxiliary information (e.g., domain, time, or labels) \cite{hyvarinen2016unsupervised,hyvarinen2017nonlinear,hyvarinen2019nonlinear}. Later works \cite{kong2022partial,xie2023multi,kong2023identification,yan2024counterfactual} relaxed these assumptions, proving identifiability without restricting latent distributions.

In unsupervised scenarios, structural priors such as mechanism sparsity are used to isolate latent factors \cite{lachapelle2023synergies,lachapelle2022partial,zhang2024causal}. Extensions to time series \cite{hyvarinen2016unsupervised,halva2020hidden,lippe2022citris,huang2023latent,yan2024counterfactual} exploit temporal nonstationarity for identifiability, with permutation-based objectives handling stationary cases. More recent methods (e.g., TDRL \cite{yao2022temporally}, LEAP \cite{yao2021learning}, CtrlNS \cite{CtrlNS}) leverage independent noise or transition sparsity to recover latent dynamics; however, models like CtrlNS remain limited to a single-layer latent structure, constraining their ability to capture multi-level temporal factors. 

\section{Preliminaries}
\subsection{Data Generation Process}
We start with the original data generation process that is shown in Figure \ref{fig:motivation} (b). The observed variables $X=[\mathbf{x}_1, \cdots, \mathbf{x}_T] \in \mathbb{R}^{T \times d}$, which also represent a feature sequence of the video with $T$ frames. We define $V=[\mathbf{v}_{1}, \cdots ,\mathbf{v}_{T} ] \in \mathbb{R}^{T \times n_v}$ as the sequence of visual latent variables, and $c=[\mathbf{c}_{1}, \cdots ,\mathbf{c}_{T'} ] \in \mathbb{R}^{T' \times n_c}$ as the sequence of action latent variables, where $T'\leq T$ since the latent action variables change more slowly than visual variables. Suppose the observed variable $\mathbf{x}_t$ is obtained by applying an invertible mixing function $g$ to visual variables $\mathbf{v}_t$, without the presence of noise. Mathematically, we have:
\begin{equation}
\small
    \label{equ:x_gen}
    \mathbf{x}_t=g(\mathbf{v}_t)
\end{equation}
% For the visual-level and $i$-th dimension latent variable $v_{t,i}$ at timestamp $t$, it is generated through a latent causal process. This process is assumed to be related to the time-delayed parent $\text{Pa}_d(v_{t,i})$ and the hierarchical parent $\text{Pa}_h(v_{t,i})$, respectively. Formally, it can be expressed using a structural equation model as follows:
\textcolor{black}{For the $i$-th dimensional latent variable $v_{t,i}$ (visual level, timestamp $t$), it is generated via a latent causal process associated with $\text{Pa}_d(v_{t,i})$ (time-delayed parent, local temporal dependencies) and $\text{Pa}_h(v_{t,i})$ (hierarchical parent, stable high-level context). Formally, this process is formulated as:}
\begin{equation}
\small
    \label{equ:visual_gen}
    v_{t,i}=f^v_i\left(\text{Pa}_d(v_{t,i}), \text{Pa}_h(v_{t,i}), \epsilon_{t,i}^v\right),\quad\text{with}\quad \epsilon_{t,i}^v\sim p_{\epsilon_{t,v}^v},
\end{equation}
where $\epsilon_{t,i}^v$ is a temporally and spatially independent noise sampled from $p_{\epsilon_{t,i}^v}$. The action variables $\mathbf{c}_t$ is only related to the time-delayed parent $\text{Pa}_d(c_{t,i})$, which is formalized as:
\begin{equation}
\small
    \label{equ:concept_gen}
    c_{t,i}=f^c_i\left(\text{Pa}_d(c_{t,i}), \epsilon_{t,i}^c\right),\quad\text{with}\quad \epsilon_{t,i}^c\sim p_{\epsilon_{t,i}^c}.
\end{equation}

\subsection{Problem Statement}
Given a $T$-frame video, the goal is to predict a sequence of action $\hat{Y}=[\hat{y}_1,\cdots, \hat{y}_T] $, where each label $\hat{y}_t \in U $ and $ U =  \{1,2,\cdots,|U|\}$ denotes the set of action categories in the dataset, including a category for the background. In the weakly supervised action segmentation, the frame-level ground-truth label sequence $Y=[y_1,\cdots,y_T]$ is not accessible during training. Instead, the model is provided with an ordered sequence of action labels, referred to as a transcript, denoted by $A = [a_1,\cdots,a_M]$, where each $a_m \in U$, and  $M$ indicates the total number of action segments in the video, including those corresponding to background.

\begin{figure*}
    \centering
    \includegraphics[width=\linewidth]{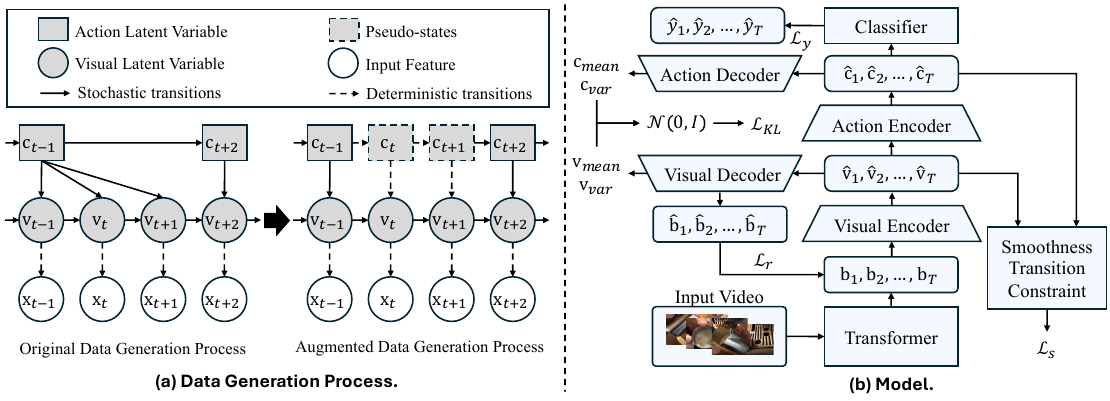}
    \caption{
    Illustration of the augmented data generation process and the framework of the \textbf{HAL} model. (a) The original data generation process is augmented by introducing pseudo-states and aligning the number of latent action with visual variables. The dashed arrows denote the unknown deterministic transitions. (b) The overall framework of \textbf{HAL} consists of a pyramidal transformer-based backbone for feature extraction, visual and action encoders for latent visual and action variables, visual and action decoder for reconstruction, and the smoothness transition constraint that enforces the identification of latent action variables.
    }
    \label{fig:Model}
\end{figure*}

\section{Approach}

Based on the aforementioned data generation process, we propose the Hierarchical Action Learning (\textbf{HAL}) model. Specifically, we first present the augmented data generation process (Figure \ref{fig:Model}(a)) for the tractability. We then detail the implementation of the HAL model in Figure \ref{fig:Model}(b).

\subsection{Augmented Data Generation Process}

As shown in Figure \ref{fig:motivation}(b), each high-level latent action variables corresponds to multiple latent visual variables. However, existing models typically encode a fixed number of latent variables, making it difficult to accommodate the original generative process directly. To bridge this gap, we propose an augmented data generation process, illustrated in Figure \ref{fig:Model}(a). This augmentation consists of two steps. First, since the number of latent action variables is smaller than that of latent visual variables, we introduce pseudo-states (the dashed boxes) to align their quantities. 
Second, to encode the prior knowledge that latent action variables evolve more slowly than latent visual variables, we model the transitions between pseudo-states as deterministic but unknown processes (dashed arrows). In this augmented causal graph, the number of latent action variables is consistent with the number of latent visual variables (i.e., video length), allowing us to leverage existing backbones for multi-level feature extraction, while preserving the prior that action transitions are smoother than those of latent visual variables. 
The aforementioned augmentation equivalently transforms a stochastic process $\fbox{$\mathbf{c}_{t-1}$}\rightarrow \fbox{$\mathbf{c}_{t+2}$}$ into a mixing process $\fbox{$\mathbf{c}_{t-1}$}\dashrightarrow \dashfbox{\mathbf{c}_{t}}\dashrightarrow \dashfbox{\mathbf{c}_{t+1}} \rightarrow \fbox{$\mathbf{c}_{t+2}$}$, where $\dashfbox{\mathbf{c}_{t}}$ and $\dashfbox{\mathbf{c}_{t+1}}$ denote the pseudo-states; $\dashrightarrow$ and $\rightarrow$ denotes the deterministic and stochastic transitions, respectively. 

Compared with stochastic transition, where extra information is brought via exogenous noise, since the transition of the pseudo-states is deterministic, no extra information is brought, so $\fbox{$\mathbf{c}_{t-1}$}$, $\dashfbox{\mathbf{c}_{t}}$, and $\dashfbox{\mathbf{c}_{t+1}}$ can be denoted as the same latent action, preserving the prior that the latent action evolve more slowly than the latent visual variables. As a result, the lengths of the latent action and visual variables are aligned, which facilitates subsequent modeling.

% \subsection{Latent Concept Learning Model}
\subsection{Hierarchical Action Learning Model}
In this section, we devise the Hierarchical Action Learning model, which is built on the pyramid transformer-based architecture with a smoothness transition constraint to enforce different changing speeds of hierarchical latent variables. %Please refer to supplementary material for more implementation of the proposed \textbf{HAL} model.
\paragraph{Video Modeling with variational inference.} 
Similar to the methods that model video or image data with the variational auto-encoder at the feature level \cite{chenlearning}, we first leverage a visual transformer-based backbone network \cite{ATBA} to extract the low-dimensional feature $\mathbf{b}_{1:T}$. Then we can take $\mathbf{x}_{1:T}$ as $\mathbf{b}_{1:T}$ and model the augmented data generation process with variational inference. We  begin by deriving the evidence lower bound (ELBO), as shown below:
\begin{equation}
\small
\begin{split}
\label{equ:elbo_full}
&ELBO=\underbrace{\mathbb{E}_{q(\mathbf{v}_{1:T}|\mathbf{b}_{1:T})}\log p(\mathbf{b}_{1:T}|\mathbf{v}_{1:T})}_{\mathcal{L}_{r}}-\\ &\underbrace{\Big( D_{KL}\big(q(\mathbf{v}_{1:T}|\mathbf{b}_{1:T})||p(\mathbf{v}_{1:T})\big) + D_{KL}\big(q(\mathbf{c}_{1:T}|\mathbf{v}_{1:T})||p(\mathbf{c}_{1:T})\big)\Big)}_{\mathcal{L}_{KL}}
\end{split},
% \vspace{-2mm}
\end{equation}
where $\mathcal{L}_r$ denotes the reconstruction loss, and $\mathcal{L}_{KL}$ denotes the Kullback–Leibler divergence. Moreover, the distributions $q(\mathbf{v}_{1:T}|\mathbf{b}_{1:T})$ and $q(\mathbf{c}_{1:T}|\mathbf{v}_{1:T})$ approximate the priors of the latent variables and correspond to the visual encoder and action encoder in Figure~\ref{fig:Model}(b), respectively. Both are implemented using Transformer architectures. And $p(\mathbf{b}_{1:T}|\mathbf{v}_{1:T})$ is used to reconstruct the feature $\mathbf{b}_{1:T}$. Mathematically, the aforementioned backbone network, visual encoder, action encoder and visual decoder, action decoder are formulated by $\phi$, $\psi$, $\eta$, $\kappa$ and $\xi$, respectively as follows:
\begin{equation}
\small
\begin{split}
    \mathbf{b}_{1:T} = \phi(\mathbf{x}_{1:T})&, \quad 
    \hat{\mathbf{v}}_{1:T} = \psi(\mathbf{b}_{1:T}), \quad 
    \hat{\mathbf{c}}_{1:T} = \eta(\hat{\mathbf{v}}_{1:T}), \\
    \hat{\mathbf{b}}_{1:T} &= \kappa(\hat{\mathbf{v}}_{1:T}), \quad 
    \hat{\mathbf{v}}_{1:T}' = \xi(\hat{\mathbf{c}}_{1:T}),
\end{split}
\end{equation}
where $\hat{\mathbf{b}}_{1:T}$ denotes the reconstructed features.

\paragraph{Smoothness Transition Constraint.} To incorporate the inductive bias that latent action variables evolve more slowly than the latent visual variables, we further propose the smoothness transition constraint. First, since the latent action and visual variables may operate on different scales, we need to apply normalization to align them onto a consistent scale. Specifically, we employ the L2 normalization on the latent action variables and the latent visual variables, which are formulated as follows: 
\begin{equation}
\small
\label{equ:L2norm}
    \overline{\mathbf{v}}_{1:T} = L2(\mathbf{v}_{1:T}),\quad \overline{\mathbf{c}}_{1:T} = L2(\mathbf{c}_{1:T}),
\end{equation}
where $\overline{\mathbf{v}}_{1:T}$ and $\overline{\mathbf{c}}_{1:T}$ denote the normalized latent visual and action variables, respectively. Sequentially, we quantify the changes of latent visual and action variables to further constraint the prior that the latent action change more slowly than latent visual variables, as shown in Equation (\ref{equ:diff})
\begin{equation}
\small
\label{equ:diff}
\begin{split}
    \Delta \overline{V} &= \{|\overline{\mathbf{v}}_2-\overline{\mathbf{v}}_1|,\cdots, |\overline{\mathbf{v}}_{T}-\overline{\mathbf{v}}_{T-1}|\} ,\\
    \Delta \overline{C} &= \{|\overline{\mathbf{c}}_2-\overline{\mathbf{c}}_1|,\cdots, |\overline{\mathbf{c}}_{T}-\overline{\mathbf{c}}_{T-1}|\} ,
\end{split}
\end{equation}
where $\Delta \overline{V}$ and $\Delta \overline{C}$ denote the changing degrees of latent visual and action variables, respectively.

As a results, we can constrain the changing of different layers of latent variables by
\begin{equation}
\small
\label{equ:smooth}
\begin{split}
    \mathcal{L}_s &= \underbrace{\textbf{ReLU}(\sum_{t=1}^{T-1} \mathbf{w}_c \Delta \overline{C} - \sum_{i=1}^{T-1} \mathbf{w}_v \Delta \overline{V})}_{\textbf{(i)}} + \underbrace{\delta \sum_{t=1}^{T-1} \mathbf{w}_c \Delta \overline{C}}_{\textbf{(ii)}},
\end{split}
\end{equation}
where $\delta$ is a hyper-parameter; $\mathbf{w}_c$ and $\mathbf{w}_v$ are the weights that make the changing parts more significant, which are calculated by a $SoftMAX$ activation as follows:
\begin{equation}
\small
\label{equ:s_loss}
      \mathbf{w}_c = \text{SoftMAX}(\Delta \overline{C}), \quad
    \mathbf{w}_v = \text{SoftMAX}(\Delta \overline{V}).  
\end{equation}
\textbf{Intuition:} Equation (\ref{equ:smooth}) implies that \textbf{1)} When the estimated latent action variables change faster than the latent visual variables, i.e., $\sum_{t=1}^{T-1} \mathbf{w}_c \Delta \overline{C} - \sum_{i=1}^{T-1} \mathbf{w}_v \Delta \overline{V} > 0$, the term (i) in Equation  (\ref{equ:smooth}) imposes a constraint to ensure that the latent action variables evolve more slowly than the latent visual variables; \textbf{2)} In the mean while, term (ii) further enforces the temporal smoothness of the latent action variables by penalizing rapid changes, encouraging consistency over time. Please refer to the Ablation Study subsection for the comprehensive effectiveness of each component of the proposed smoothness transition constraint.

\paragraph{Model Summary} By combining the aforementioned ELBO, the smoothness transition constraint with the classifier loss for segmentation, we can finally formalize the total loss of the proposed \textbf{HAL} model as follows:
\begin{equation}
\small
\begin{split}
\label{equ:total_loss}
    \mathcal{L}_{total} &= \mathcal{L}_{y} - \alpha \cdot ELBO + \beta \cdot \mathcal{L}_{s},
\end{split}
\end{equation}
where we follow the total loss as $\mathcal{L}_{y}$ from ATBA \cite{ATBA}, $\alpha$ and $\beta$ are hyper-parameters.

\subsection{Comparison with Existing Methods}
Although existing methods for action segmentation also leverage techniques like smoothness or hierarchical modeling, we would like to highlight that the proposed method is different from the existing methods in the following ways: 1) Compared with the existing smoothness-based methods like Gaussian Smoothing \cite{du2022fast} and Boundary Smoothing \cite{ATBA}, our method constrains the smoothness on the latent variables, i.e., the latent action variables instead of the predicted label. 2) Compared with the methods like \cite{sarfraz2021temporally,mounir2023streamer} that also model videos in a hierarchical manner, our method can show that the high-level latent action variables can be identified with theoretical guarantees.  

\section{Theoretical Analysis}

To show the theoretical guarantees, we begin with the definition of block-wise identification and linear operators.

\begin{definition}

[\textbf{Block-wise Identifiability of Latent Action $\mathbf{c}_t$ and Visual Variables $\mathbf{v}_t$} \cite{von2021self}] 
\label{eq:block_iden}
The block-wise identifiability of $(\mathbf{v}_t, \mathbf{c}_t) \in \mathbb{R}^{n_v+n_c}$ means that for ground-truth $(\mathbf{v}_t, \mathbf{c}_t)$, there exists $(\hat{\mathbf{v}}_t,\hat{\mathbf{c}}_t) \in \mathbb{R}^{n_v+n_c}$ and an invertible function $h_{\mathbf{v}, \mathbf{c}}:\mathbb{R}^{n_v+n_c}\rightarrow\mathbb{R}^{n_v+n_c}$, such that $(\mathbf{v}_t,\mathbf{c}_t)=h(\hat{\mathbf{v}}_t,\hat{\mathbf{c}}_t)$.
\end{definition}
Although the above definition is stated for $(\mathbf{v}_t, \mathbf{c}_t)$, it naturally extends to the case of latent action variables $\mathbf{c}_t$. Specifically, when we say that the latent action variables $\mathbf{c}_t \in \mathbb{R}^{n_c}$ is block-wise identifiable (as shown in Theorem \ref{theorem:2}), we mean that there exists an estimated $\hat{\mathbf{c}}_t \in \mathbb{R}^{n_c}$ and an invertible function $h_c:\mathbb{R}^{n_c}\rightarrow \mathbb{R}^{n_c}$ such that $\mathbf{c}_t=h_c(\hat{\mathbf{c}}_t)$.

\begin{definition} [\textbf{Linear Operator} \cite{hu2008instrumental,dunford1988linear}] \label{Defn:linear}
Consider two random variables $\mathbf{a}$ and $\mathbf{b}$ with support $\mathcal{A}$ and $\mathcal{B}$, the linear operator $L_{\mathbf{b}|\mathbf{a}}$ is defined as a mapping from a probability function $p_{\mathbf{a}}$ in some function space $\mathcal{F}(\mathcal{A})$ onto the probability function $p_{\mathbf{b}}=L_{\mathbf{b}|\mathbf{a}}\circ p_{\mathbf{a}}$ in some function space $\mathcal{F}(\mathcal{B})$,
\begin{equation}
% \vspace{-2mm}
\small
    \mathcal{F}(\mathcal{A})\rightarrow \mathcal{F}(\mathcal{B}): p_{\mathbf{b}}=L_{\mathbf{b}|\mathbf{a}}\circ p_{\mathbf{a}}=\int_{\mathcal{A}} p_{\mathbf{b}|\mathbf{a}}(\cdot|\mathbf{a})p_{\mathbf{a}}(\mathbf{a})d\mathbf{a} .
\end{equation}
\vspace{-5mm}
\end{definition}
Intuitively, the definition of the linear operator describes a transformation from one probability function $p_{a}(a)$ to another probability function $p_{b}(b)$. Please refer to supplementary material for more discussion, explanation, and examples of the linear operator. Based on the aforementioned definitions, we first show that the hierarchical latent variables $(\mathbf{v}_t, \mathbf{c}_t)$ are block-wise identifiable in Lemma \ref{theorem:1}.

\subsection{Identifiability of Action and Visual Variables}
\begin{lemma}
% \label{the:hu}
\label{theorem:1}
\begin{figure*}
    \centering
    \includegraphics[width=0.95\linewidth]{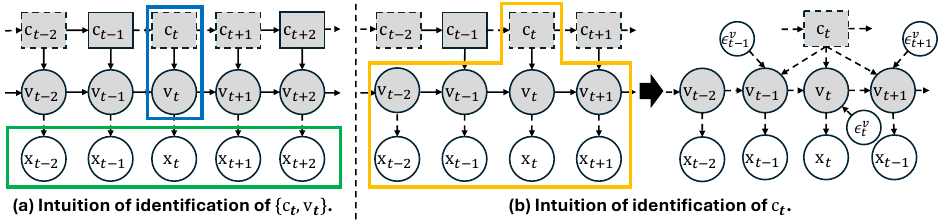}
    \caption{Intuitions of the theoretical results, where the solid and dashed arrows denote the stochastic and deterministic transitions, respectively. (a) The identification of $(\mathbf{v}_t, \mathbf{c}_t)$ can be achieved by leveraging five consecutive temporal observations. (b) By introducing the independent noises $\epsilon^v_{t-1}, \epsilon^v_{t}, \epsilon^v_{t+1}$, the stochastic transition processes $\mathbf{v}_{t-2} \rightarrow \mathbf{v}_{t-1}$ and $\mathbf{v}_{t} \rightarrow \mathbf{v}_{t+1}$ can be effectively transformed into corresponding deterministic transition processes. }
    
    % \vspace{-2mm}
    \label{fig:intuition}
\end{figure*}

(\textbf{Block-wise Identification of $(\mathbf{v}_t, \mathbf{c}_t)$.}) Suppose the observed, latent visual, and latent action variables follow the augmented data generation process in Figure \ref{fig:Model}(a). By matching the true joint distribution of 5 numbers of adjacent video frames, i.e., $\{\mathbf{x}_{t-2},\mathbf{x}_{t-1},\mathbf{x}_t,\mathbf{x}_{t+1},\mathbf{x}_{t+2}\}$, we further make the following assumptions:
\begin{itemize}[leftmargin=*,  itemsep=5pt]  
    \item A1 \label{A1} \underline{(\textbf{Bounded and Continuous Density}):} The joint distribution of $\mathbf{x}_t$,$\mathbf{v}_t$,$\mathbf{c}_t$, and their marginal and conditional densities are bounded and continuous.
    \item A2 \label{A2} \underline{(\textbf{Injective Linear Operators}):} The linear operators $L_{\mathbf{x}_{t+1} | \mathbf{v}_{t}, \mathbf{c}_{t}}$ and $L_{\mathbf{x}_{t-1} | \mathbf{x}_{t+1}}$ are injective for bounded function space.
    \item A3 \label{A3} \underline{ (\textbf{Positive Density}):}For all $[\mathbf{v}_{t}, \mathbf{c}_{t}],[\mathbf{v}_{t}, \mathbf{c}_{t}]' \in \mathcal{V}_{t} \cup \mathcal{C}_{t}$ with $[\mathbf{v}_{t}, \mathbf{c}_{t}] \neq [\mathbf{v}_{t}, \mathbf{c}_{t}]'$ the set $\{ \mathbf{x}_{t}:p(\mathbf{x}_{t}|\mathbf{v}_{t}, \mathbf{c}_{t}) \neq p(\mathbf{x}_{t}| \mathbf{v}_{t}', \mathbf{c}_{t}') \}$ has positive probability.
\end{itemize}
Suppose that the learned $(\hat{g}, \hat{f}, p_{\hat{\epsilon}})$ to achieve Equation (\ref{equ:x_gen}) - (\ref{equ:concept_gen}), then the latent visual and action variables $(\mathbf{v}_t, \mathbf{c}_t)$ are block-wise identifiable.
\end{lemma}

\paragraph{Intuition and Proof Sketch:} 

Please refer to the detailed proof in supplementary material. This theoretical extension builds upon the work of \cite{hu2008instrumental}, where the latent variables are assumed to be single-layered. Under this assumption, three consecutive observations are sufficient to characterize the dynamics of the latent variable $(\mathbf{v}_t, \mathbf{c}_t)$. In contrast, our model involves a two-layer latent structure, where additional observations, i.e., $\mathbf{x}_{t-2}$ and $\mathbf{x}_{t+2}$, are required to identify $(\mathbf{v}_t, \mathbf{c}_t)$. Intuitively, these five observations effectively capture the underlying temporal evolution of latent variables, as shown in Figure~\ref{fig:intuition}(a).

The proof can be separated into two steps. First, we can construct an eigenvalue-eigenfunction decomposition regarding the integral operator by leveraging the definition of linear operators. Sequentially, using the uniqueness of spectral decomposition (Theorem XV.4.3.5 \cite{dunford1988linear}), latent variables $(\mathbf{v}_t, \mathbf{c}_t)$ are block-wise identifiable when the marginal distribution of observations is matched.

\subsection{Identifiability of Latent Action Variables}
Based on the identifiability of $(\mathbf{v}_t, \mathbf{c}_t)$, we further leverage the hierarchical property of latent variables to show that the latent action $\mathbf{c}_t$ are block-wise identifiable, i.e., the estimated $\hat{\mathbf{c}}_t$ only contains the information of true $\mathbf{c}_t$.

\begin{theorem}
\label{theorem:2}
(\textbf{Block-wise Identifiability of Latent Action Variables.}) Suppose assumptions for Lemma \ref{theorem:1} hold and for each action-level block, it has at least two latent visual variables as descendants in data generation process in Figure \ref{fig:intuition}(b), we further make the following assumptions:
\begin{itemize}[leftmargin=*,  itemsep=5pt]  
    \item A4 \label{A4} \underline{(\textbf{non-singular Jacobian}):} Each $f_m$ has non-singular Jacobian matrices almost anywhere and $f_m$ is invertible.
    \end{itemize}
Then by matching the marginal distribution of observed variables and modeling the hierarchical relationships, the latent action variables can be block-wise identifiable.
\end{theorem}

\paragraph{Proof Sketch:} 
The proof can be found in supplementary material. The intuition is shown in Figure \ref{fig:intuition} (b), where we turn the stochastic transition processes $\mathbf{v}_{t-2} \rightarrow \mathbf{v}_{t-1}$ and $\mathbf{v}_{t} \rightarrow \mathbf{v}_{t+1}$ into the deterministic transition processes by introducing independent noise, i.e., $(\mathbf{v}_{t-2}, \epsilon^v_{t-1}) \dashrightarrow \mathbf{v}_{t-1}$ and $(\mathbf{v}_{t}, \epsilon^v_{t+1}) \dashrightarrow \mathbf{v}_{t+1}$. Sequentially, since the mixing procedure from $\mathbf{v}_t$ to $\mathbf{x}_t$ is assumed to be deterministic, $\mathbf{v}_{t-1}$ and $\mathbf{v}_{t+1}$ are independent when $\mathbf{x}_t$ and $\mathbf{c}_t$ are conditioned. As a result, we can leverage the fact that $\epsilon^v_{t-1}, \mathbf{c}_t$, and $\epsilon^v_{t+1}$ are conditionally independent to prove that $\mathbf{c}_t$ are block-wise identifiable. Based on this intuition, we can draw the proof sketch as follows. 
First, we construct an invertible transformation between the estimated variables $(\hat{\mathbf{c}}_t, \hat{\mathbf{v}}_{t-2},\hat{\mathbf{v}}_{t},\hat{\epsilon}^v_{t-1}, \hat{\epsilon}^v_{t+1})$ and the ground-truth variables $(\mathbf{c}_t,\mathbf{v}_{t-2},\mathbf{v}_{t}, \epsilon^v_{t-1}, \epsilon^v_{t+1})$ by leveraging the deterministic transition. Sequentially, by leveraging the independence between $\mathbf{c}_t$ and $\mathbf{v}_{t-2},\mathbf{v}_{t}, \epsilon^v_{t-1}, \epsilon^v_{t+1}$, we can show that the  action variables $\mathbf{c}_t$ is not the function of $\hat{\mathbf{v}}_{t-2},\hat{\mathbf{v}}_{t}, \hat{\epsilon}^v_{t-1}$ and $\hat{\epsilon}^v_{t+1}$, implying that $\mathbf{c}_t$ is only the function of $\hat{\mathbf{c}}_t$ and $\mathbf{c}_t$ are block-wise identifiable. The identification results can be found in the supplementary material.

\section{Experiments}
\subsection{Datasets}
We consider the following benchmark datasets for weakly-supervised action segmentation. The \textbf{Breakfast} dataset \cite{kuehne2014language} contains approximately 77 hours of video material and covers 48 distinct action categories. The \textbf{CrossTask} dataset \cite{zhukov2019cross} contains 4.7K teaching videos (lasting 374 hours), covering 83 tasks and areas such as cooking and car repair. \textbf{Hollywood Extended} dataset \cite{bojanowski2014weakly} contains 937 video clips from 69 Hollywood movies. Each clip is labeled with the action sequence, and the clips are divided into time intervals of 10 frames in length. The \textbf{GTEA} dataset \cite{gtea} contains 28 videos collected from 4 participants, annotated with 11 fine-grained action categories(e.g., take, spread, pour, and stir). We follow the setting of ATBA \cite{ATBA}. Specifically, for the Breakfast dataset, we use the provided 4 training/test splits and report the mean and standard deviation. 
For the Hollywood dataset, we conduct a 10-fold cross-validation. For the CrossTask dataset, we utilize the released training/test split. For the GTEA dataset, we use the 4 training/test splits.

\subsection{Baselines}
We consider different types of baselines. As for the Viterbi-based method, we choose HMM+RNN \cite{richard2017weakly}, CDFL \cite{CDFL}, TASL \cite{TASL}, and NN-Viterbi \cite{richard2018neuralnetwork}. For the Dynamic Time Warping (DTW)-based method, we consider D3TW \cite{chang2019d3tw} and DP-DTW \cite{DPDTW}. Additionally, we evaluate other representative approaches such as TCFPN+ISBA \cite{ding2018weakly}, MuCon \cite{Mucon}, POC \cite{POC}, and AdaAct \cite{AdaAct}. We also consider the recent advanced framework like CtrlNS \cite{CtrlNS}.
% Moreover, we also consider the smoothness-based method, like boundary smoothness \cite{ATBA} and Gaussian smoothness \cite{du2022fast}. 
% For a fair comparison, the method based on Gaussian smoothness adopts the backbone networks of \cite{ATBA} and is named ATBA+GS.

\subsection{Evaluation Metric}
The evaluation metrics employed in the test include Mean-over-Frames (\textbf{MoF}), Mean-over-Frames without Background (\textbf{MoF-bg}), Intersection-over-Union (\textbf{IoU}), and Intersection-over-Detection (\textbf{IoD}). The definitions of IoU and IoD are formulated as follows:$| I\bigcap I^*|/|I\bigcup I^* |$ and $| I\bigcap I^*|/|I|$, where $I^*$ and $I$ are the ground-truth (GT) segmentation and the predicted segmentation. The definitions of IoU and IoD follow ISBA \cite{ding2018weakly}.

\begin{table}[t]
\caption{Experiment results on Breakfast dataset.}
\vspace{-1em}
\label{tab:breakfast_exp}
\resizebox{\columnwidth}{!}{%
\begin{tabular}{@{}ccccc@{}}
\toprule
\multicolumn{5}{c}{\textbf{Breakfast}}                                                                                     \\ \midrule
\multicolumn{1}{c|}{Method}                                     & MoF               & MoF-Bg            & IoU               & IoD               \\ \midrule
\multicolumn{1}{c|}{HMM+RNN\cite{richard2017weakly}}            & 33.3              & -                 & -                 & -                 \\
\multicolumn{1}{c|}{TCFPN+ISBA\cite{ding2018weakly}}            & 38.4/36.4±1.0     & 38.4              & 24.2              & 40.6              \\
\multicolumn{1}{c|}{NN-Viterbi\cite{richard2018neuralnetwork}}  & 43.0/39.7±2.4     & -                 & -                 & -                 \\
\multicolumn{1}{c|}{D3TW\cite{chang2019d3tw}}                   & 45.7              & -                 & -                 & -                 \\
\multicolumn{1}{c|}{CDFL\cite{CDFL}}                            & 50.2/48.1±2.5     & 48.0              & 33.7              & 45.4              \\
\multicolumn{1}{c|}{DP-DTW\cite{DPDTW}}                         & 50.8              & -                 & 35.6              & 45.1              \\
\multicolumn{1}{c|}{TASL\cite{TASL}}                            & 47.8              & -                 & 35.2              & 46.1              \\
\multicolumn{1}{c|}{MuCon\cite{Mucon}}                          & 48.5±1.8          & 50.3              & 40.9              & 54.0              \\
\multicolumn{1}{c|}{POC\cite{POC}}                              & 45.7              & -                 & 38.3              & 46.4              \\
\multicolumn{1}{c|}{AdaAct\cite{AdaAct}}                        & 51.2              & 48.3              & 36.3              & 46.4              \\
\multicolumn{1}{c|}{ATBA\cite{ATBA}}                            & 53.9±1.2          & 54.4±1.2          & 41.1±0.7          & 61.7±1.1          \\  \midrule
\multicolumn{1}{c|}{HAL(Ours)}                                  & \textbf{56.3}±1.3 & \textbf{57.2}±1.6                                      & \textbf{42.6}±1.9 & \textbf{62.4}±2.5 \\ \bottomrule
\end{tabular}%
}
\end{table}

\begin{table}[t]
\caption{Experiment results on the CrossTask dataset.}
\vspace{-1em}
\label{tab:Crosstask_exp}
\resizebox{\columnwidth}{!}{%
\begin{tabular}{@{}ccccc@{}}
\toprule
\multicolumn{5}{c}{\textbf{CrossTask}}                                                                                     \\ \midrule
\multicolumn{1}{c|}{Method}                                         & MoF               & MoF-Bg            & IoU               & IoD               \\ \midrule
\multicolumn{1}{c|}{NN-Viterbi\cite{richard2018neuralnetwork}}      & 26.5              & -                 & 10.7              & 24.0              \\
\multicolumn{1}{c|}{CDFL\cite{CDFL}}                                & 31.9              & -                 & 11.5              & 23.8              \\
\multicolumn{1}{c|}{TASL\cite{TASL}}                                & 40.7              & 27.4              & 14.5              & 25.1              \\
\multicolumn{1}{c|}{POC\cite{POC}}                                  & 42.8              & 17.6              & 15.6              & -                 \\
\multicolumn{1}{c|}{ATBA\cite{ATBA}}                                & 50.6±1.3          & 31.3±0.7          & 20.9±0.4          & \textbf{44.6}±0.7 \\
\multicolumn{1}{c|}{CtrlNS\cite{CtrlNS}}                            & \textbf{54.0}±0.9 & -                 & 15.7±0.5          & 23.6±0.8          \\ \midrule
\multicolumn{1}{c|}{\textbf{HAL(Ours)}}                                      & \textbf{54.0}±0.8 & \textbf{35.0}±1.1 & \textbf{21.6}±0.4 & 41.9±0.7          \\ \bottomrule

\end{tabular}%
}
\end{table}

\begin{table}[t]
\caption{Experiment results on the Hollywood dataset.}
\vspace{-1em}
\label{tab:Hollywood_exp}
\resizebox{\columnwidth}{!}{%
\begin{tabular}{@{}ccccc@{}}
\toprule
\multicolumn{5}{c}{\textbf{Hollywood Extended}}                                                                          \\ \midrule
\multicolumn{1}{c|}{Method}                                 & MoF               & MoF-Bg            & IoU               & IoD               \\ \midrule
\multicolumn{1}{c|}{HMM+RNN \cite{richard2017weakly}}       & -                 & -                 & 11.9              & -                 \\
\multicolumn{1}{c|}{TCFPN+ISBA \cite{ding2018weakly}}       & 28.7              & 34.5              & 12.6              & 18.3              \\
\multicolumn{1}{c|}{D3TW \cite{chang2019d3tw}}              & 33.6              & -                 & -                 & -                 \\
\multicolumn{1}{c|}{CDFL \cite{CDFL}}                       & 45.0              & 40.6              & 19.5              & 25.8              \\
\multicolumn{1}{c|}{TASL \cite{TASL}}                       & 42.1              & 27.2              & 23.3              & 33.0              \\
\multicolumn{1}{c|}{MuCon \cite{Mucon}}                     & -                 & 41.6              & 12.9              & -                 \\
\multicolumn{1}{c|}{ATBA \cite{ATBA}}                       & 47.7±2.5          & 40.2±1.6          & 30.9±1.6          & 55.8±0.8          \\
\multicolumn{1}{c|}{CtrlNS \cite{CtrlNS}}                   & \textbf{52.9}±3.1 & -                 & 32.7±1.3          & 52.4±1.8          \\ \midrule
\multicolumn{1}{c|}{\textbf{HAL(Ours)}}                               & 51.0±1.7          & \textbf{41.9}±3.7 & \textbf{33.4}±1.4 & \textbf{56.8}±1.8 \\ \bottomrule
\end{tabular}%
}
\end{table}

\begin{table}[]
\caption{Experiment results on the GTEA dataset.}
\vspace{-1em}
\label{tab:gtea}
\resizebox{\columnwidth}{!}{%
\begin{tabular}{@{}ccccc@{}}
\toprule
\multicolumn{5}{c}{\textbf{GTEA}}   \\ \midrule
\multicolumn{1}{c|}{Method}                                     & MoF               & MoF-Bg                                  & IoU               & IoD               \\ \midrule
\multicolumn{1}{c|}{NN-Viterbi\cite{richard2018neuralnetwork} $\sharp$}  & 25.7±2.5  & 25.9±2.8                              & 13.0±1.8  & 18.9±2.3  \\
\multicolumn{1}{c|}{CDFL\cite{CDFL} $\sharp$}                   & 31.4±3.1  & 31.9±4.0                                       & 15.5±1.4  & 22.8±2.6  \\
\multicolumn{1}{c|}{POC\cite{POC} $\sharp$}                     & 18.4±4.5  & 22.2±6.1                                       & 7.3±2.7   & 11.2±3.6  \\
\multicolumn{1}{c|}{ATBA\cite{ATBA} $\sharp$}                   & 42.2±4.3  & 42.7±5.6                                       & 24.4±1.9  & 47.2±4.1  \\
\multicolumn{1}{c|}{CtrlNS\cite{CtrlNS} $\sharp$}               & 33.3±3.9  & 33.1±5.9                                       & 18.9±2.1  & 38.3±4.4  \\ \midrule
\multicolumn{1}{c|}{\textbf{HAL(Ours)}}                                  & \textbf{45.2}±5.4 & \textbf{46.0}±6.6                      & \textbf{25.6}±2.2 & \textbf{49.2}±2.7 \\ \bottomrule
\end{tabular}
}
\end{table}

\subsection{Qualitative Results and Discussion}
\paragraph{Comparison with Existing Methods.}

Experimental results on the Breakfast, CrossTask, and Hollywood datasets are presented in Table \ref{tab:breakfast_exp}, \ref{tab:Crosstask_exp}, \ref{tab:Hollywood_exp} and \ref{tab:gtea}. The results of the Breakfast, CrossTask, and Hollywood datasets are reported by \cite{ATBA}. $\sharp$-The reported results are obtained by us via rerunning the open sources. Best results are in bold.

The proposed HAL model consistently surpasses all baselines across most action segmentation metrics, demonstrating its strong effectiveness under weak supervision. Notably, HAL achieves the highest scores on the IoU and IoD metrics across most of the datasets, highlighting its ability to capture structured, hierarchical action representations. The slightly lower IoD performance on the CrossTask and MoF on the Hollywood datasets can be attributed to the greater diversity and complexity of background scenes, which pose additional challenges for accurate background modeling.
The performance gains of our method mainly stem from the proposed smoothness regularization term, $\mathcal{L}_s$, which enforces temporal consistency in high-level action transitions. This constraint promotes better alignment between the estimated and ground-truth segmentations, leading to notable improvements in IoU and IoD. While HAL achieves comparable MoF results to ATBA and CtrlNS, its advantage lies in segmenting videos based on hierarchical action variables rather than raw visual features. By mitigating the influence of transient visual variations and irrelevant background noise, our approach yields more accurate and semantically coherent segmentation results.

\subsection{Case Studies}

We further provide qualitative case studies of action segmentation in Figure~\ref{fig:case_study}. From the visualization results, several observations can be made. (1) Our proposed HAL model exhibits stable and coherent segmentation even in complex action spaces, such as those in the Breakfast dataset. Compared with representative methods like ATBA and CtrlNS, HAL maintains smoother temporal boundaries and avoids the frequent boundary oscillations commonly observed in prior models. (2) As illustrated in Figure~\ref{fig:case_study}, segmentation based solely on visual representations (e.g., HAL-V) often produces fragmented and abrupt transitions, leading to inconsistent action boundaries and noisy temporal patterns. In contrast, our hierarchical approach, which segments based on high-level action variables, yields results that align much more closely with the ground truth. The smoother and semantically consistent transitions in our predictions highlight the benefit of modeling slowly evolving hierarchical action, which effectively suppress spurious fluctuations in the visual domain and enhance the temporal coherence of the final segmentation.

\begin{figure}
    \centering
    \includegraphics[width=\linewidth]{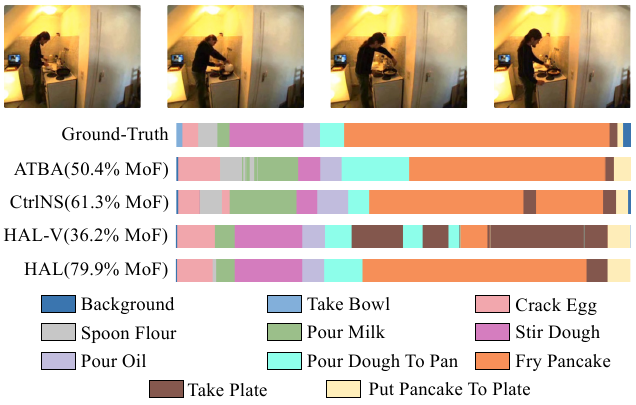}
    \caption{Qualitative results of \textit{P04-cam01-P04-pancake} on the Breakfast dataset.}
    \label{fig:case_study}
\end{figure}

\begin{table}[t]
\small
\caption{Ablation studies on the split 1 of Breakfast dataset.}
\vspace{-1em}
\label{tab:ablation}
\resizebox{\columnwidth}{!}{%
\begin{tabular}{@{}cccccccc@{}}
\toprule
\multicolumn{1}{c|}{EXP}    & $\mathcal{L}_{r}$ & $\mathcal{L}_{s}$     & $\mathcal{L}_{KL}$    & \multicolumn{1}{c|}{$\delta$}     & MoF           & IoU           & IoD           \\ \midrule
\multicolumn{1}{c|}{1}      & -                 & -                     & -                     & \multicolumn{1}{c|}{-}            & 53.3          & 40.1          & 58.7          \\
\multicolumn{1}{c|}{2}      & \ding{51}         & -                     & -                     & \multicolumn{1}{c|}{-}            & 54.3          & 38.4          & 61.6          \\
\multicolumn{1}{c|}{3}      & -                 & \ding{51}             & -                     & \multicolumn{1}{c|}{-}            & 54.6          & 40.3          & 61.6          \\
\multicolumn{1}{c|}{4}      & -                 & -                     & \ding{51}             & \multicolumn{1}{c|}{-}            & 54.5          & 42.0          & 61.0          \\
\multicolumn{1}{c|}{5}      & -                 & -                     & -                     & \multicolumn{1}{c|}{\ding{51}}    & 53.9          & 41.1          & 59.4          \\ \midrule
\multicolumn{1}{c|}{6}      & \ding{51}         & \ding{51}            & \ding{51}              & \multicolumn{1}{c|}{\ding{51}}    & \textbf{56.6} & \textbf{42.6} & \textbf{62.1} \\ \bottomrule
\end{tabular}%
}
\end{table}

\subsection{Ablation Study}
To evaluate the contribution of individual components, we conducted several experiments, summarized in Table~\ref{tab:ablation}, to assess both their separate and combined effects. The evaluation metrics include MoF, IoU, and IoD. Results from Experiments 1–5 demonstrate a steady improvement in MoF, confirming the positive impact of each component. Incorporating both $\mathcal{L}_r$ and $\mathcal{L}_{KL}$ further enhances performance, emphasizing the importance of modeling video distributions. However, adding $\mathcal{L}_r$ alone (Exp. 2) leads to overemphasis on feature reconstruction, causing overly expanded segmentations and a decrease in IoU. In Experiment 6, combining all components achieves the highest scores across all metrics, validating the effectiveness of the full framework and showing that integration strengthens the model while enabling more robust segmentation.

\begin{figure}
    \centering
    \includegraphics[width=\linewidth]{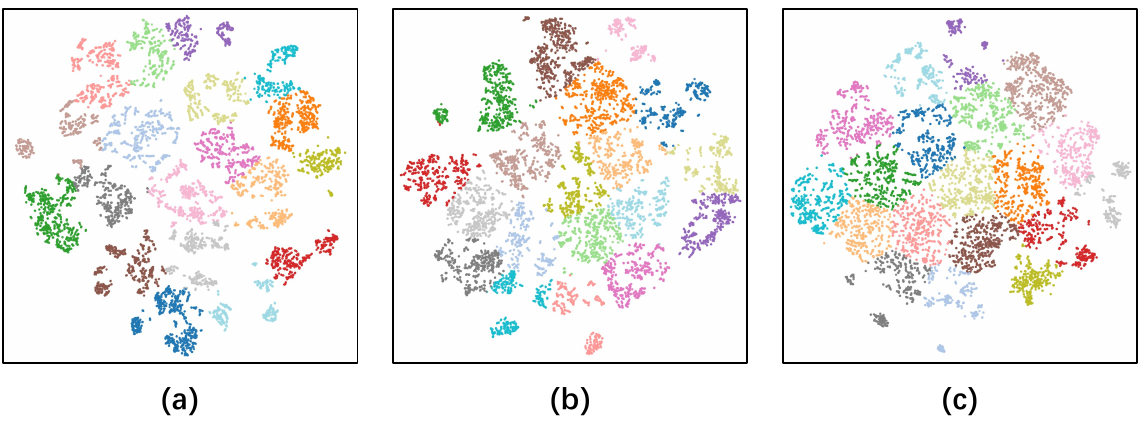}
    \caption{Illustration of T-SNE Visualization of latent variables. (a) and (b) show scatter plots of the latent action and visual variables of the propsoed method, respectively. (c) shows the output of ATBA after dimensionality reduction.}
    \label{fig:tsne}
\end{figure}

\subsection{Visualization Results}
To further show the effectiveness of hierarchical latent variables, we further provide T-SNE visualization results in Figure~\ref{fig:tsne}, where Figure~\ref{fig:tsne} (a) and (b) denote the visualization results from latent action variables and visual variables, respectively. Figure~\ref{fig:tsne} (c) illustrates the results of the ATBA model. In order to ensure a fair comparison, we visualized the same number of layers in the ATBA model as in our model's high-level action variables. According to the visualization results, we observe the following conclusions: 1) Compared to ATBA, our method produces denser clusters, indicating that our results are more stable and compact latent variable representations. This denser clustering suggests our model effectively avoids over-segmentation by leveraging high-level action structures to reduce noise from visual features. 2) High-level latent variables show greater cohesion than low-level variables. This demonstrates that high-level action variables are more stable and align better with ground-truth segmentation, validating the effectiveness of our approach in capturing hierarchical dynamics.

\section{Summary}
This paper proposes a Hierarchical Action Learning (\textbf{HAL}) model for weakly supervised action segmentation, addressing over-segmentation and noisy boundaries that arise from reliance on low-level visual cues. Unlike existing approaches, \textbf{HAL} introduces a hierarchical causal framework in which slowly evolving high-level Action variables govern the dynamics of rapidly changing visual variables. Technically, the model combines an augmented generative process, a pyramidal transformer, and temporal smoothness constraints to disentangle Action and visual representations and capture multi-scale temporal structure. Theoretically, the model guarantees identifiability of high-level Action variables under mild assumptions. Overall, this work provides both a principled and effective approach to improving weakly supervised action segmentation through hierarchical causal reasoning. Extending this idea to video generation will be an interesting future direction.

\section{Acknowledgment}
This research was supported in part by National Science and Technology Major Project (2021ZD0111502),  Natural Science Foundation of China (U24A20233) and  CCF-DiDi GAIA Collaborative Research Funds (CCF-DiDi GAIA 202521).

{
    \small
    \bibliographystyle{ieeenat_fullname}
    \bibliography{main}
}

\clearpage
\setcounter{page}{1}
\maketitlesupplementary

\section{\textcolor{black}{Extended Related Work}}\label{app:related}

\subsection{Weakly-Supervised Action Segmentation}
Weakly-supervised action segmentation approaches \cite{bojanowski2014weakly,connectionistTemporal,kuehne2017weakly, RNN-HMM,richard2017weakly, CDFL, TASL,richard2018neuralnetwork,chang2019d3tw, DPDTW, AdaAct} seek to reduce the dependence on detailed frame-level annotations.  In the context of action segmentation, four primary forms of weak supervision have been studied: transcripts, action sets, timestamps, and textual descriptions (e.g., narrations). Among them, transcripts provide an action sequence without frame-level alignment, while action sets offer an unordered collection of actions appearing in the video. Timestamps indicate the occurrence of specific actions at particular time points.

A transcript provides a sequential list of actions occurring in a video, but does not include precise frame-level timing information. This type of supervision offers substantial advantages in terms of annotation efficiency, as it removes the requirement for detailed, frame-wise labeling. Approaches utilizing transcripts generally adopt either a single-stage design or an iterative two-step learning paradigm.  \textbf{Iterative two-stage methods} first generate initial frame-level labels from the action transcript, then refine them through repeated updates. HTK \cite{kuehne2017weakly} adapts a supervised HMM-GMM model \cite{kuehne2016end} to weak supervision by uniformly initializing segments and refining them using transcripts. Richard et al. \cite{richard2017weakly,RNN-HMM} replace GMMs with RNNs and introduce latent sub-actions for finer motion modeling. ISBA \cite{ding2018weakly} starts with uniform segmentation and gradually adjusts boundaries using soft labels. TASL \cite{TASL} iteratively aligns video frames with transcripts to improve segmentation. \textbf{Single-stage methods} aim to overcome the initialization sensitivity of two-stage approaches. ECTC \cite{connectionistTemporal} aligns transcripts with video frames using temporal classification while enforcing consistency. NN-Viterbi \cite{richard2018neuralnetwork} integrates visual, context, and length models, generating pseudo-labels via Viterbi decoding. D3TW \cite{chang2019d3tw} introduces a differentiable loss to distinguish between correct and incorrect transcripts. Building on this, CDFL \cite{CDFL} uses a segmentation graph to model valid/invalid alignments and compares their energies. Despite strong performance, NN-Viterbi and CDFL incur high computational costs. To address this, MuCon \cite{Mucon} reduces training time with a dual-branch design enforcing mutual consistency.  DP-DTW  \cite{DPDTW} learns class-specific prototypes and improves action discrimination for temporal recognition.

Action sets specify only the presence of actions in a video, without indicating their order or frequency, making them a weaker supervision form than transcripts. Early work by Richard et al. \cite{richard2018action} models temporal segmentation under action set supervision by integrating contextual, duration, and action models to infer the most probable action sequence for a given video, formulated as a maximum likelihood problem solvable using the Viterbi algorithm. Later methods, such as SCT \cite{fayyaz2020sct}, directly learn segmentation from action sets through region-wise predictions and frame-level consistency constraints. SCV \cite{li2020set} enhances pseudo-label generation via a constrained Viterbi approach and an $n$-pair loss. ACV \cite{li2021anchor} extends SCV by proposing a differentiable formulation that enables end-to-end training and eliminates the reliance on post-processing steps. POC \cite{POC} proposes a loss formulation that leverages the cross-video consistency of action pair ordering to enforce temporal alignment between the predicted segments and the
extracted templates.

Timestamp supervision provides sparse frame annotations instead of dense labels, offering a cost-effective alternative to full supervision. Some methods \cite{li2021temporal,khan2022timestamp}  assume one timestamp per action, while others \cite{rahaman2022generalized} support arbitrary sampling. Typically, these approaches generate pseudo frame-wise labels that are refined iteratively. Li et al. \cite{li2021temporal} detect action transitions and apply a confidence loss to guide predictions around timestamps. EM-TSS \cite{rahaman2022generalized} employs an EM framework to infer full sequences from sparse labels, offering greater flexibility in annotation and outperforming approaches constrained to one timestamp per action. GCN-TSS \cite{khan2022timestamp} introduces a GNN-based approach where video frames are modeled as graph nodes, with edges weighted by feature similarity. Labels from sparsely annotated frames are propagated across the graph to infer full segmentations, assuming all action instances are timestamped. In contrast, RAS-TSS \cite{souri2022robust} relaxes this assumption by allowing incomplete annotations and expands segment boundaries around timestamps to reduce ambiguity and better handle missing labels.

Narrations and subtitles, often available alongside instructional videos, offer a valuable but weak supervisory signal for temporal understanding. Prior work has leveraged such text for alignment \cite{malmaud2015s,bojanowski2015weakly} or step localization \cite{zhukov2019cross,alayrac2016unsupervised}  and recent efforts in multi-modal learning \cite{radford2021learning,shvetsova2022everything} explore joint embeddings across vision, language, and occasionally audio. However, a key limitation is the frequent misalignment between textual and visual content. To address this, Sener et al. \cite{sener2015unsupervised} propose a hybrid generative model combining visual and textual vocabularies across videos for shared action discovery, while Fried et al. \cite{fried2020learning} assume a canonical transcript per activity class and use a semi-Markov model for structured alignment. More recent approaches \cite{han2022temporal} focus on aligning narrations to specific segments within instructional videos.

\subsection{Causal Representation Learning} 

To recover latent variables with theoretical guarantees of identifiability \cite{yao2023multi,scholkopf2021toward,liu2023causal,gresele2020incomplete}, Independent Component Analysis (ICA) has been widely employed for uncovering causal representations \cite{rajendran2024learning,mansouri2023object,wendong2024causal,li2024identification}. Traditional ICA approaches typically rely on the assumption that observed variables are generated through a linear mixture of latent  variables \cite{comon1994independent,hyvarinen2013independent,lee1998independent,zhang2007kernel}.  However, identifying such a linear mixing function is  nontrivial in practical, real-world settings. To establish identifiability in more general nonlinear ICA models, additional assumptions are introduced, such as sparsity in the data generation process and the incorporation of auxiliary variables \cite{zheng2022identifiability,hyvarinen1999nonlinear,hyvarinen2024identifiability,khemakhem2020ice,li2023identifying}. 

In particular, the work of Aapo et al. provides the first theoretical foundation for identifiability in such settings \cite{khemakhem2020variational}. When latent sources are assumed to follow distributions from the exponential family, identifiability can be achieved by integrating auxiliary information such as domain indices, temporal markers, or class labels \cite{khemakhem2020variational,hyvarinen2016unsupervised,hyvarinen2017nonlinear, hyvarinen2019nonlinear}. Moreover, recent studies \cite{kong2022partial, xie2023multi,kong2023identification,yan2024counterfactual} demonstrate that the identifiability of individual components in nonlinear ICA can be achieved without relying on exponential family assumptions. 

To enable identifiability in unsupervised scenarios, many approaches have introduced structural assumptions such as sparsity in the generative mechanisms \cite{zheng2022identifiability,hyvarinen1999nonlinear,hyvarinen2024identifiability,khemakhem2020ice,li2023identifying}. For instance, Lachapelle et al. \cite{lachapelle2023synergies, lachapelle2022partial} propose mechanism sparsity regularization as an inductive bias to isolate distinct causal latent factors. Zhang et al. \cite{zhang2024causal} utilize sparse latent structures and establish identifiability under conditions of distributional shifts, without requiring explicit supervision. Besides, nonlinear ICA is extended to the time-series setting for learning identifiable representations from temporal data \cite{hyvarinen2016unsupervised,yan2024counterfactual,huang2023latent,halva2020hidden,lippe2022citris}. For example, Aapo et al. \cite{hyvarinen2016unsupervised} propose a theoretical framework based on nonstationary variances across time segments to identify latent sources in temporally varying data. In the stationary case, researchers utilize permutation-based contrastive learning methods to disentangle latent components.

More recent advances incorporate independent noise and variability history information. Importantly, TDRL~\cite{yao2022temporally}  and LEAP~\cite{yao2021learning} exploit such features to facilitate identifiability. CHiLD~\cite{child}, a recently proposed framework for hierarchical temporal causal representation learning, leverages temporal contextual observations and hierarchical sparsity to achieve identifiability of multi-layer latent dynamics. By introducing a nonstationary sparse transition assumption,  CtrlNS~\cite{CtrlNS} establishes identifiability from a theoretical perspective and empirically demonstrates its framework’s effectiveness in uncovering latent factors and distribution shifts. Although CtrlNS can also be applied to the action segmentation task, its latent variables are limited to a single layer, which restricts its ability to capture the varying dynamics across different levels of latent factors.

% % proof
% #############################
\section{Proof}
\label{app:proof}
\subsection{Useful Theorems and Lemmas}
\begin{theorem}(Theorem XV 4.5 in \citep{dunford1988linear} Part $\uppercase\expandafter{\romannumeral 3}$) \label{thm:dunford}
    A bounded operator \( T \) is a spectral operator if and only if it is the sum \( T = S + N \) of a bounded scalar type operator \( S \) and a quasi-nilpotent operator \( N \) commuting with \( S \). Furthermore, this decomposition is unique and \( T \) and \( S \) have the same spectrum and the same resolution of the identity.
    \end{theorem}
\begin{lemma}(Lemma 1 in~\citep{hu2008instrumental})
    \label{label:lemma}
    Under Assumption A2, if $L_{z|x}$ is injective, then $L_{x|z}^{-1}$ exists and is densely defined over $\mathscr{G}(\mathscr{X})$ (for $\mathscr{G} = \mathscr{L}^1, \mathscr{L}^1_{\mathrm{bnd}}$).
    \end{lemma}

\subsection{Proof of Block-wise Identifiability of Latent Action and Visual Variables.}
\begin{lemma}
(\textbf{Block-wise Identification of $(\mathbf{v}_t, \mathbf{c}_t)$.}) Suppose the observed, latent visual, and latent action variables follow the augmented data generation process in Figure \ref{fig:Model}(a). By matching the true joint distribution of 5 numbers of adjacent video frames, i.e., $\{\mathbf{x}_{t-2},\mathbf{x}_{t-1},\mathbf{x}_t,\mathbf{x}_{t+1},\mathbf{x}_{t+2}\}$, we further make the following assumptions:
\begin{itemize}[leftmargin=*,  itemsep=5pt]  
    \item A1 \underline{(\textbf{Bounded and Continuous Density}):} The joint distribution of $\mathbf{x}_t$,$\mathbf{v}_t$,$\mathbf{c}_t$, and their marginal and conditional densities are bounded and continuous.
    \item A2 \underline{(\textbf{Injective Linear Operators}):} The linear operators $L_{\mathbf{x}_{t+1} | \mathbf{v}_{t}, \mathbf{c}_{t}}$ and $L_{\mathbf{x}_{t-1} | \mathbf{x}_{t+1}}$ are injective for bounded function space.
    \item A3 \underline{ (\textbf{Positive Density}):}For all $[\mathbf{v}_{t}, \mathbf{c}_{t}],[\mathbf{v}_{t}, \mathbf{c}_{t}]' \in \mathcal{V}_{t} \cup \mathcal{C}_{t}$ with $[\mathbf{v}_{t}, \mathbf{c}_{t}] \neq [\mathbf{v}_{t}, \mathbf{c}_{t}]'$ the set $\{ \mathbf{x}_{t}:p(\mathbf{x}_{t}|\mathbf{v}_{t}, \mathbf{c}_{t}) \neq p(\mathbf{x}_{t}| \mathbf{v}_{t}', \mathbf{c}_{t}') \}$ has positive probability.
\end{itemize}
Suppose that the learned $(\hat{g}, \hat{f}, p_{\hat{\epsilon}})$ to achieve Equation (\ref{equ:x_gen}) - (\ref{equ:concept_gen}), then the latent variables $(\mathbf{v}_t, \mathbf{c}_t)$ are block-wise identifiable.
\end{lemma}

\begin{proof} We first follow Hu et al~\cite{hu2008instrumental} framework to prove that $(\mathbf{v}_t, \mathbf{c}_t)$ is block-wise identifiable given sufficient observation. Sequentially, we prove that we require at least $5$ adjacent observed variables to achieve block-wise identifiability.

Given time series data with $T$ timesteps $X=\{\mathbf{x}_1,\cdots,\mathbf{x}_t,\cdots,\mathbf{x}_T\}$, $2$-layers of latent variables $V = \{\mathbf{v}_1, \cdots, \mathbf{v}_t, \cdots, \mathbf{v}_T\}$ and $C = \{\mathbf{c}_1, \cdots, \mathbf{c}_t, \cdots, \mathbf{c}_T\}$. To simplify the notation, we let $\mathbf{x}_{<t}$, $\mathbf{x}_{>t}$ and $\mathbf{z}_t$ be $\{\mathbf{x}_{t-2},\mathbf{x}_{t-1}\}$, $\{\mathbf{x}_{t+1},\mathbf{x}_{t+2}\}$ and $(\mathbf{v}_t,\mathbf{c}_t)$, respectively. Sequentially, according to the data generation process in Figure \ref{fig:intuition}(a), we have:
\begin{equation}
\begin{split}
\label{the1_1}
    P(\mathbf{x}_{<t}|\mathbf{x}_{t},\mathbf{z}_t) &= P(\mathbf{x}_{<t}|\mathbf{z}_t),  \\
    \quad P(\mathbf{x}_{>t}|\mathbf{x}_t, \mathbf{x}_{<t},\mathbf{z}_t) &= P(\mathbf{x}_{>t}|\mathbf{z}_t).
\end{split}
\end{equation}
Sequentially, the observed $P(\mathbf{x}_{t-1})$ and joint distribution $P(\mathbf{x}_{>t},\mathbf{x}_t,\mathbf{x}_{<t})$ directly indicates $P(\mathbf{x}_{>t}, \mathbf{x}_t|\mathbf{x}_{<t})$, and we have:
\begin{equation}
\label{the1_2}
\begin{split}
\begin{split}
    &P(\mathbf{x}_{>t},\mathbf{x}_t|\mathbf{x}_{<t})\\  
    &= \underbrace{\int_{\mathcal{Z}_t}P(\mathbf{x}_{>t},\mathbf{x}_t,\mathbf{z}_t|\mathbf{x}_{<t}) d\mathbf{z}_t}_{\text{Integration over $\mathcal{Z}_t$}} \\ 
    &= \underbrace{\int_{\mathcal{Z}_t}P(\mathbf{x}_{>t}|\mathbf{x}_t,\mathbf{z}_t,\mathbf{x}_{<t})P(\mathbf{x}_t,\mathbf{z}_t|\mathbf{x}_{<t})d\mathbf{z}_t}_{\text{Factorization of joint conditional probability}} \\
    &=\underbrace{\int_{\mathcal{Z}_t}P(\mathbf{x}_{>t}|\mathbf{z}_t)P(\mathbf{x}_t,\mathbf{z}_t|\mathbf{x}_{<t})d\mathbf{z}_t}_{\text{Conditional Independence}} \\
    &=\underbrace{\int_{\mathcal{Z}_t}P(\mathbf{x}_{>t}|\mathbf{z}_t)P(\mathbf{x}_{t}|\mathbf{z}_t,\mathbf{x}_{<t})P(\mathbf{z}_t|\mathbf{x}_{<t})d\mathbf{z}_t}_{\text{Bayes Law}} \\
    &=\int_{\mathcal{Z}_t}P(\mathbf{x}_{>t}|\mathbf{z}_t)P(\mathbf{x}_t|\mathbf{z}_t)P(\mathbf{z}_t|\mathbf{x}_{<t})d\mathbf{z}_t.
\end{split}
\end{split}
\end{equation}
We further incorporate the integration over $\mathcal{X}_{<t}$ as follows:
\begin{equation}
\begin{split}
    \label{the1_3}
&\int_{\mathcal{X}_{<t}}P(\mathbf{x}_{>t},\mathbf{x}_t|\mathbf{x}_{<t})P(\mathbf{x}_{<t})d\mathbf{x}_{<t} \\ &=\int_{\mathcal{X}_{<t}}\int_{\mathcal{Z}_t}P(\mathbf{x}_{>t}|\mathbf{z}_t)P(\mathbf{x}_t|\mathbf{z}_t)P(\mathbf{z}_t|\mathbf{x}_{<t})P(\mathbf{x}_{<t})d\mathbf{z}_t d\mathbf{x}_{<t}.
\end{split}
\end{equation}
According to the definition of linear operator, we have:
\begin{equation}
\label{the1_4}
\begin{split}
\int_{\mathcal{X}_{<t}}P(\mathbf{x}_{>t},\mathbf{x}_t|\mathbf{x}_{<t})P(\mathbf{x}_{<t})d\mathbf{x}_{<t}&=[L_{\mathbf{x}_{>t},\mathbf{x}_t|\mathbf{x}_{<t}}\circ P](\mathbf{x}_{<t}),\\
\int_{\mathcal{X}_{<t}}P(\mathbf{z}_{t}|\mathbf{x}_{<t})P(\mathbf{x}_{<t})d\mathbf{x}_{<t} &= [L_{\mathbf{z}_t|\mathbf{x}<t}\circ P](\mathbf{x}_{<t})\\
\int_{\mathcal{Z}_{t}} P(\mathbf{x}_{>t}|\mathbf{z}_{t})d\mathbf{z}_t &=L_{\mathbf{x}_{>t}|\mathbf{z}_t}.
\end{split}
\end{equation}

By combining Equation (\ref{the1_3}) and (\ref{the1_4}), we have:
\begin{equation}
    [L_{\mathbf{x}_{>t},\mathbf{x}_t|\mathbf{x}_{<t}}\circ P](\mathbf{x}_{<t})=[L_{\mathbf{x}_{>t}|\mathbf{z}_t}D_{\mathbf{x}_t|\mathbf{z}_t}L_{\mathbf{z}_t|\mathbf{x}<t}\circ P](\mathbf{x}_{<t}),
\end{equation}
which implies the operator equivalence:
\begin{equation}
\label{the1_5.5}
    L_{\mathbf{x}_{>t},\mathbf{x}_t|\mathbf{x}_{<t}} = L_{\mathbf{x}_{>t}|\mathbf{z}_t}D_{\mathbf{x}_t|\mathbf{z}_t}L_{\mathbf{z}_t|\mathbf{x}<t}.
\end{equation}
Sequentially, we further integrate out $\mathbf{x}_t$ and have:
\begin{equation}
    \int_{\mathcal{X}_t}  L_{\mathbf{x}_{>t},\mathbf{x}_t|\mathbf{x}_{<t}} d\mathbf{x}_t = \int_{\mathcal{X}_t}L_{\mathbf{x}_{>t}|\mathbf{z}_t}D_{\mathbf{x}_t|\mathbf{z}_t}L_{\mathbf{z}_t|\mathbf{x}<t} d\mathbf{x}_t,
\end{equation}
and it results in:
\begin{equation}
\label{the1_6}
    L_{\mathbf{x}_{>t}|\mathbf{x}_{<t}} = L_{\mathbf{x}_{>t}|\mathbf{z}_t} L_{\mathbf{z}_t|\mathbf{x}<t}.
\end{equation}
According to assumption A2, the linear operator $L_{\mathbf{x}_{>t}|\mathbf{z}_t}$ is injective, Equation (\ref{the1_6}) can be rewritten as:
\begin{equation}
\label{the1_7}
    L_{\mathbf{x}>t|\mathbf{z}_t}^{-1}L_{\mathbf{x}_{>t}|\mathbf{x}_{<t}} = L_{\mathbf{z}_t|\mathbf{x}<t}.
\end{equation}
By combining Equation (\ref{the1_5.5}) and (\ref{the1_7}), we have
\begin{equation}
    L_{\mathbf{x}_{>t},\mathbf{x}_t|\mathbf{x}<t} = L_{\mathbf{x}>t|\mathbf{z}_t}D_{\mathbf{x}_t|\mathbf{z}_t}L_{\mathbf{x}>t|\mathbf{z}_t}^{-1}L_{\mathbf{x}>t|\mathbf{x}<t}.
\end{equation}
By leveraging Lemma \ref{label:lemma}, if $L_{\mathbf{x}_{<t}|\mathbf{x}_{>t}}$ is injective, then $L_{\mathbf{x}_{>t}|\mathbf{x}_{<t}}^{-1}$ exists. Therefore, we have:
\begin{equation} \label{equ:trueLDL}
    L_{\mathbf{x}_{>t},\mathbf{x}_t|\mathbf{x}_{<t}} L_{\mathbf{x}>t|\mathbf{x}<t}^{-1} = L_{\mathbf{x}>t|\mathbf{z}_t}D_{\mathbf{x}_t|\mathbf{z}_t}L_{\mathbf{x}>t|\mathbf{z}_t}^{-1}.
\end{equation}
Then we can leverage assumption A3 and the linear operator is bounded. Consequently, $L_{\mathbf{x}_{>t},\mathbf{x}_t|\mathbf{x}_{<t}} L_{\mathbf{x}>t|\mathbf{x}<t}^{-1}$ is also bounded, which satisfies the condition of Theorem \ref{thm:dunford}, and hence the the operator $L_{\mathbf{x}>t|\mathbf{z}_t}D_{\mathbf{x}_t|\mathbf{z}_t}L_{\mathbf{x}>t|\mathbf{z}_t}^{-1}$ have a unique spectral decomposition, where $L_{\mathbf{x}>t|\mathbf{z}_t}$ and $D_{\mathbf{x}_t|\mathbf{z}_t}$ correspond to eigenfunctions and
eigenvalues, respectively.

Since both the marginal and conditional distributions of the observed variables are matched, the true model and the estimated model yield the same distribution over the observed variables. Therefore, we also have:
\begin{equation}
\label{equ:true_estimated}
    L_{\mathbf{x}_{>t},\mathbf{x}_t|\mathbf{x}_{<t}} L^{-1}_{\mathbf{x}>t|\mathbf{x}<t} = L_{\hat{\mathbf{x}}_{>t},\hat{\mathbf{x}}_t|\hat{\mathbf{x}}_{<t}} L^{-1}_{\hat{\mathbf{x}}>t|\hat{\mathbf{x}}<t},  
\end{equation}
where the L.H.S corresponds to the true model and the R.H.S corresponds to the estimated model. Moreover, $L_{\hat{\mathbf{x}}_{>t},\hat{\mathbf{x}}_t|\hat{\mathbf{x}}_{<t}} L^{-1}_{\hat{\mathbf{x}}>t|\hat{\mathbf{x}}<t}$ also have the unique decomposition, so the L.H.S of the Equation (\ref{equ:true_estimated}) can be written as:
\begin{equation} \label{equ:estLDL}
    L_{\mathbf{x}_{>t},\mathbf{x}_t|\mathbf{x}_{<t}} L^{-1}_{\mathbf{x}>t|\mathbf{x}<t} = L_{\hat{\mathbf{x}}_{>t}|\hat{\mathbf{z}}_{t}}D_{\hat{\mathbf{x}}_t|\hat{\mathbf{z}}_t}L^{-1}_{\hat{\mathbf{x}}_{>t}|\hat{\mathbf{z}}_t},
\end{equation}
Integrating Equation (\ref{equ:trueLDL}) and Equation (\ref{equ:estLDL}), and noting that their L.H.S. are identical, it follows that they share the same spectral decomposition. This yields
\begin{equation}
    L_{\mathbf{x}_{>t}|\mathbf{z}_t} = C L_{\hat{\mathbf{x}}_{>t}|\hat{\mathbf{z}}_t} P, \quad D_{\mathbf{x}_t|\mathbf{z}_t} = P^{-1} D_{\hat{\mathbf{x}}_t|\hat{\mathbf{z}}_t} P,
\end{equation}
where $C$ is a scalar accounting for scaling indeterminacy and $P$ is a permutation on the order of elements in $D_{\hat{\mathbf{x}}_t|\hat{\mathbf{z}}_t}$, as discussed in~\citep{dunford1988linear}. These forms of indeterminacy are analogous to those in eigendecomposition, which can be viewed as a finite-dimensional special case. \textcolor{black}{P is a mapping from distribution to distribution}

Since the normalizing condition 
\begin{equation}
    \int_{\hat{\mathcal{X}}_{t+1}} p_{\hat{\mathbf{x}}_t|\hat{\mathbf{z}}_t} \, d\hat{\mathbf{x}}_t = 1
\end{equation}
must hold for every $\hat{\mathbf{z}}_t$, one only solution is to set $C=1$. 

Hence, $D_{\hat{\mathbf{x}}_t|\hat{\mathbf{z}}_t}$ and $D_{\mathbf{x}_t|\mathbf{z}_t}$ are identical up to a permutation on their repsective elements. We use unordered sets to express this equivalence:
\begin{equation}
\{ p(\mathbf{x}_t \mid \mathbf{z}_t) \} = \{ p(\mathbf{x}_t \mid \hat{\mathbf{z}}_t) \}, \quad \text{for all } \mathbf{z}_t, \hat{\mathbf{z}}_t.
\end{equation}
Due to the set being unordered, the only way to match the R.H.S. with the L.H.S. in a consistent order is to exchange the conditioning variables, that is, 
\begin{equation}
\begin{split}
    &\Big\{ p(\mathbf{x}_t|\mathbf{z}_t^{(1)}),p(\mathbf{x}_t|\mathbf{z}_t^{(2)}),\cdots \Big\} \\
    &= \Big\{ p(\mathbf{x}_t|\hat{\mathbf{z}}_t^{(\pi(1))}),p(\mathbf{x}_t|\hat{\mathbf{z}}_t^{(\pi(2))}),\cdots \Big\},
\end{split}
\end{equation}
where superscript $(\cdot)$ denotes the index of a conditioning variable, and $\pi$ is reindexing the conditioning variables. We use a relabeling map \(h\) to represent its corresponding value mapping:
\begin{equation}
\label{equ:D_equ3}
    p(\mathbf{x}_t|\mathbf{z}_t) = p(\mathbf{x}_t|h(\hat{\mathbf{z}}_t)), \text{ for all } \mathbf{z}_t,\hat{\mathbf{z}}_t
\end{equation}

Since $K_{\hat{\mathbf{z}}_t,\mathbf{z}_t}, L^{-1}_{{\mathbf{x}}_{>t}|\hat{\mathbf{z}}_t}$, and $L_{\mathbf{z}_t|\mathbf{x}_{>t}}$ are continuous, $h$ is continuous and differentiable. Moreover, by leveraging assumption A3, different values of $\mathbf{z}$, i.e., $\mathbf{z}_t^{(1)}, \mathbf{z}_t^{(2)}$ imply $p(\mathbf{x}_t|\mathbf{z}_t^{(1)})\neq p(\mathbf{x}_t|\mathbf{z}_t^{(2)})$. So we can construct a function $F:\mathcal{Z}\rightarrow p(\mathbf{x}_t|\mathbf{z}_t)$, and we have:
\begin{equation}
    \mathbf{z}_t^{(1)} \neq \mathbf{z}_t^{(2)} \longrightarrow F(\mathbf{z}_t^{(1)}) \neq F(\mathbf{z}_t^{(2)}),
\end{equation}
implying that $F$ is injective. Moreover, by using Equation (\ref{equ:D_equ3}), we have $F(\mathbf{z}_t)=F(h(\hat{\mathbf{z}}_t))$, which implies $\mathbf{z}_t=h(\hat{\mathbf{z}}_t)$.

The aforementioned result leverage $\mathbf{x}_{<t}, \mathbf{x}_t$, and $\mathbf{x}_{>t}$ as three different measurement of $\mathbf{z}_t$, where $|\mathbf{x}_{<t}| \mathcal{g} |\mathbf{z}_t|$, $|\mathbf{x}_{>t}| \mathcal{g} |\mathbf{z}_t|$ and $|\mathbf{x}_t| < |\mathbf{z}_t|$. It may imply that when the $\mathbf{x}_t$ cannot provide enough information to recover $\mathbf{z}_t$, we can seek more information from $\mathbf{x}_{<t}$ and $\mathbf{x}_{>t}$. 

Sequentially, we further prove that when the observed and $2$-layer latent variables follow the data generation process in Equation (\ref{equ:x_gen}) and (\ref{equ:visual_gen}), we require at least $5$ adjacent observed variables, i.e., $\{\mathbf{x}_{t-2},\mathbf{x}_{t-1},\mathbf{x}_t,\mathbf{x}_{t+1},\mathbf{x}_{t+2}\}$ to make $(\mathbf{v}_t,\mathbf{c}_t)$ block-wise identifiable. We prove it by contradiction as follows.

Suppose we have $4$ adjacent observations, which can be divided into two cases: 1) $\{\mathbf{x}_{t-1},\mathbf{x}_t,\mathbf{x}_{t+1},\mathbf{x}_{t+2}\}$ and $\{\mathbf{x}_{t-2},\mathbf{x}_{t-1},\mathbf{x}_t,\mathbf{x}_{t+1}\}$. In the first case, suppose the dimension of $\mathbf{x}_t$ and that of latent variables $\mathbf{z}_t=(\mathbf{v}_t, \mathbf{c}_t)$ are $2n$, the dimension of $\mathbf{x}_{t-1}$ is $n$ and the dimension of $\mathbf{z}_t$ is $2n$, conflicting with the assumption that $L_{\mathbf{x}_{t+1},\mathbf{x}_{t+2}|\mathbf{z}_t}$ is injective. In the second case, the dimensions of $\mathbf{x}_{t-2},\mathbf{x}_{t-1}$ and $\mathbf{x}_{t+1}$ are $2n$ and $n$, respectively, conflicting with the assumption that $L_{\mathbf{x}_{t-2},\mathbf{x}_{t-1}|\mathbf{x}_{t+1}}$ is injective. As a result, we require at least $5$ adjacent observations, i.e., $\{\mathbf{x}_{t-2},\mathbf{x}_{t-1},\mathbf{x}_t,\mathbf{x}_{t+1},\mathbf{x}_{t+2}\}$ to make $(\mathbf{v}_t,\mathbf{c}_t)$ block-wise identifiable. In fact, as long as the total dimension of the observed variable $\mathbf{x}_{t-\tau},\cdots,\mathbf{x}_{t-1}$ is greater than or equal to the total dimension of the latent variable $(\mathbf{v}_t,\mathbf{c}_t)$, the mapping from the latent variable to the observed variable is injective. Then the latent variable $(\mathbf{v}_t,\mathbf{c}_t)$ is identifiable.
\end{proof}

\subsection{Proof of Block-wise Identifiability of Latent Action variables.}
In this section, we present the proof for Theorem \ref{theorem:2}. We first give a general identifiability theory (i.e., Theorem \ref{theorem:generative_identifiability}) for the generating process in Figure \ref{fig:intuition}(b) and then make the connection to the proof of Theorem \ref{theorem:2}.

\begin{theorem}
\label{theorem:generative_identifiability}
The generating process in Figure is defined as follows:
\begin{align}
\label{eq:general_invertible_process}
    [\mathbf{v}_{t-1}, \mathbf{v}_{t+1}] &= f^v(\mathbf{c}_{t}, \mathbf{v}_{t-2},\mathbf{v}_{t},\epsilon^v_{t-1},\epsilon^v_{t+1}), \\
    \mathbf{v}_{t-1} & = f^v_{t-1} (\mathbf{c}_{t},\mathbf{v}_{t-2},\epsilon^v_{t-1}), \\
    \mathbf{v}_{t+1} & = f^v_{t+1} (\mathbf{c}_{t},\mathbf{v}_{t},\epsilon^v_{t+1}),
\end{align}
where $ \mathbf{c}_{t}\in\mathbb{R}^{n_c} $, $ \epsilon^v_{t-1}\in\mathbb{R}^{n_{\epsilon}} $, and $ \epsilon^v_{t+1} \in\mathbb{R}^{n_{\epsilon}}$. 
    Both $f^{v}_{t-1}$ and $f^{v}_{t+1}$ are smooth and have non-singular Jacobian matrices almost anywhere, and $f$ is invertible.

If $\hat{f}^v_{t-1}$ and $\hat{f}^v_{t+1}$ assume the generating process of the true model $(f^v_{t-1}, f^v_{t+1})$ and match the joint distribution $p_{\mathbf{v}_{t-1}, \mathbf{v}_{t+1}}$, then $\mathbf{c}_{t}$ are block-wise identifiable.
\end{theorem}

\begin{proof}
To simplify the notation, we use $f$, $f_{t-1}$, $f_{t+1}$, $\epsilon_{t-1}$ and $\epsilon_{t+1}$ represent $f^v$, $f^v_{t-1}$, $f^v_{t+1}$, $\epsilon^v_{t-1}$ and $\epsilon^v_{t+1}$, respectively.

For $(\mathbf{v}_{t-1}, \mathbf{v}_{t+1}) \sim p_{\mathbf{v}_{t-1}, \mathbf{v}_{t+1}} $, because of the matched joint distribution, we have the following relations between the true variables $ (\mathbf{c}_{t},\mathbf{v}_{t-2},\mathbf{v}_{t},\epsilon_{t-1},\epsilon_{t+1}) $ and the estimated ones $ (\hat{\mathbf{c}}_{t},\hat{\mathbf{v}}_{t-2},\hat{\mathbf{v}}_{t},\hat{\epsilon}_{t-1}, \hat{\epsilon}_{t+1}) $:
\begin{align}
    &\mathbf{v}_{t-1} = f_{t-1} (\mathbf{c}_{t},\mathbf{v}_{t-2}, \epsilon_{t-1}) = \hat{f}_{t-1} ( \hat{\mathbf{c}}_{t}, \hat{\mathbf{v}}_{t-2}, \hat{\epsilon}_{t-1} ), \label{eq:vt-1_gen}	\\
    &\mathbf{v}_{t+1} = f_{t+1} (\mathbf{c}_{t},\mathbf{v}_{t}, \epsilon_{t+1}) = \hat{f}_{t+1} ( \hat{\mathbf{c}}_{t},\hat{\mathbf{v}}_{t}, \hat{\epsilon}_{t+1} ) , \label{eq:vt+1_gen} \\
    \begin{split}
        &(\hat{\mathbf{c}}_{t},\hat{\mathbf{v}}_{t-2}, \hat{\mathbf{v}}_{t}, \hat{\epsilon}_{t-1}, \hat{\epsilon}_{t+1}) 
    \\&=\hat{f}^{-1} (\mathbf{v}_{t-1}, \mathbf{v}_{t+1}) 
    \\&= \hat{f}^{-1} ( f(\mathbf{c}_{t}, \mathbf{v}_{t-2}, \mathbf{v}_{t},\epsilon_{t-1},\epsilon_{t+1}) )
    \\&= h(\mathbf{c}_{t},\mathbf{v}_{t-2}, \mathbf{v}_{t}, \epsilon_{t-1},\epsilon_{t+1}), 
    \label{eq:h_func} 
    \end{split}
\end{align}
where $\hat{f}_{t-1}$, $\hat{f}_{t+1}$ are the estimated invertible transition function and $h:= \hat{f}^{-1} \circ f$ denotes a smooth and invertible function that transforms the true variables $\mathbf{c}_{t},\mathbf{v}_{t-2},\mathbf{v}_{t}, \epsilon_{t-1},\epsilon_{t+1}$ to the estimated ones $\hat{\mathbf{c}}_{t},\hat{\mathbf{v}}_{t-2}, \hat{\mathbf{v}}_{t}, \hat{\epsilon}_{t-1}, \hat{\epsilon}_{t+1}$.

By combining Equation (\ref{eq:h_func})  into Equation (\ref{eq:vt-1_gen}) yields the following:

\begin{align}  
    \begin{split}
        &f_{t-1} ( \mathbf{c}_{t}, \mathbf{v}_{t-2}, \epsilon_{t-1} ) 
        \\&= \hat{f}_{t-1} ( h_{\mathbf{c}_{t},\mathbf{v}_{t-2}, \epsilon_{t-1}} (\mathbf{c}_{t},\mathbf{v}_{t-2},\mathbf{v}_{t},\epsilon_{t-1},\epsilon_{t+1}) ).
        \label{eq:vt-1_esimated_to_true}
    \end{split}
\end{align}

For $ i \in \{1, \dots, n_v \} $ and $ j \in \{ 1,\dots, n_{\epsilon}\}$, we take partial derivative of the $ i $-th dimension of $\mathbf{v}_{t}$ on both sides of Equation (\ref{eq:vt-1_esimated_to_true}) w.r.t. $ \epsilon_{t+1, j} $ and have:

\begin{align}
    \begin{split}
                0 &= \frac{ \partial f_{t-1, i} ( \mathbf{c}_{t},\mathbf{v}_{t-2}, \epsilon_{t-1} )}{ \partial \epsilon_{t+1,j} } 
                \\&= \frac{ \partial \hat{f}_{t-1, i} (h_{\mathbf{c}_{t},\mathbf{v}_{t-2}, \epsilon_{t-1}} (\mathbf{c}_{t},\mathbf{v}_{t-2}, \mathbf{v}_{t}, \epsilon_{t-1},\epsilon_{t+1}) ) }{ \partial \epsilon_{t+1,j} }.
    \end{split}
    \label{eq:first_partial_de}
\end{align}
The equation equals zero because there is no $ \epsilon_{t+1, j} $ in the left-hand side of the equation.
Expanding the derivative on the right-hand side gives:

\begin{equation}
     \sum_{l \in \{ 1, \dots, n_c +n_v+ n_{\epsilon} \} } \frac{ \partial \hat{f}_{t-1, i} }{ \partial h_{(\mathbf{c}_{t},\mathbf{v}_{t-2}, \epsilon_{t-1}), l} } \cdot \frac{ \partial h_{(\mathbf{c}_{t}, \mathbf{v}_{t-2},\epsilon_{t-1}), l} }{ \partial \epsilon_{t+1,j} } = 0.
    \label{eq:one_equation_1}
\end{equation}

Since $\hat{f}_{t-1}$ is invertible, the determinant of $\textbf{J}_{\hat{f}_{t-1}}$ does not equal to 0, meaning that for $n_c+n_v+n_{\epsilon}$ different values of $\hat{f}_{t-1,i}$, each vector $ [ \frac{ \partial \hat{f}_{t-1, i} }{ \partial h_{(\mathbf{c}_{t},\mathbf{v}_{t-2}, \epsilon_{t-1}), 1} }, \dots, \frac{ \partial \hat{f}_{t-1, i} }{ \partial h_{(\mathbf{c}_{t},\mathbf{v}_{t-2}, \epsilon_{t-1}), n_c+n_v+n_{\epsilon}} }]$ are linearly independent. Therefore, the $(n_c+n_v+n_{\epsilon})\times(n_c+n_v+n_{\epsilon})$ linear system is invertible and has the unique solution as follows:
\begin{equation}
\label{eq:linear_dep_1}
    \frac{ \partial h_{(\mathbf{c}_{t},\mathbf{v}_{t-2}, \epsilon_{t-1}), l} }{ \partial \epsilon_{t+1,j} } ( \mathbf{c}_{t},\mathbf{v}_{t-2},\mathbf{v}_{t},\epsilon_{t-1},\epsilon_{t+1} ) = 0.
\end{equation}

According to Equation (\ref{eq:linear_dep_1}), for any $l \in \{1, \dots, n_c + n_v + n_{\epsilon} \}$ and $j \in \{1, \dots, n_{\epsilon} \}$, $h_{(\mathbf{c}_{t}, \mathbf{v}_{t-2},\epsilon_{t-1}), l}( \mathbf{c}_{t},\mathbf{v}_{t-2},\mathbf{v}_{t},\epsilon_{t-1},\epsilon_{t+1} )$ does not depend on $\epsilon_{t+1,j}$. In other word, $(\mathbf{c}_{t}, \mathbf{v}_{t-2}, \epsilon_{t-1})$, does not depend on $ \epsilon_{t+1}$.

For $ i \in \{1, \dots, n_v \} $ and $ k \in \{ 1,\dots, n_v \}$, we take partial derivative of the $ i $-th dimension of $\mathbf{v}_{t}$ on both sides of Equation (\ref{eq:vt-1_esimated_to_true}) w.r.t. $ \mathbf{v}_{t, k} $ and have:
\begin{align}
    \begin{split}
                0 &= \frac{ \partial f_{t-1, i} ( \mathbf{c}_{t},\mathbf{v}_{t-2}, \epsilon_{t-1} )}{ \partial \mathbf{v}_{t,k} } 
                \\&= \frac{ \partial \hat{f}_{t-1, i} (h_{\mathbf{c}_{t},\mathbf{v}_{t-2}, \epsilon_{t-1}} (\mathbf{c}_{t},\mathbf{v}_{t-2}, \mathbf{v}_{t}, \epsilon_{t-1},\epsilon_{t+1}) ) }{ \partial \mathbf{v}_{t,k}  }
                \\&=\sum_{l \in \{ 1, \dots, n_c +n_v+ n_{\epsilon} \} } \frac{ \partial \hat{f}_{t-1, i} }{ \partial h_{(\mathbf{c}_{t},\mathbf{v}_{t-2}, \epsilon_{t-1}), l} } \cdot \frac{ \partial h_{(\mathbf{c}_{t}, \mathbf{v}_{t-2},\epsilon_{t-1}), l} }{ \partial \mathbf{v}_{t,k} } = 0.
    \end{split}
    \label{eq:first_partial_de}
\end{align}
That means that for $n_c+n_v+n_{\epsilon}$ different values of $\hat{f}_{t-1,i}$, each vector $ [ \frac{ \partial \hat{f}_{t-1, i} }{ \partial h_{(\mathbf{c}_{t},\mathbf{v}_{t-2}, \epsilon_{t-1}), 1} }, \dots, \frac{ \partial \hat{f}_{t-1, i} }{ \partial h_{(\mathbf{c}_{t},\mathbf{v}_{t-2}, \epsilon_{t-1}), n_c+n_v+n_{\epsilon}} }]$ are linearly independent. Therefore, the $(n_c+n_v+n_{\epsilon})\times(n_c+n_v+n_{\epsilon})$ linear system is invertible. In other word, $(\mathbf{c}_{t}, \mathbf{v}_{t-2}, \epsilon_{t-1})$, does not depend on $ \mathbf{v}_{t}$.

Similarly, by combining Equation (\ref{eq:h_func}) and (\ref{eq:vt+1_gen}), we have:
\begin{align}
    \begin{split}
        \label{eq:vt+1_esimated_to_true}
        &f_{t+1} ( \mathbf{c}_{t},\mathbf{v}_{t}, \epsilon_{t+1} ) 
        \\&= \hat{f}_{t+1} ( h_{\mathbf{c}_{t}, \mathbf{v}_{t},\epsilon_{t+1}} (\mathbf{c}_{t},\mathbf{v}_{t-2},\mathbf{v}_{t},\epsilon_{t-1},\epsilon_{t+1}) ).
    \end{split}
\end{align}

For $ i \in \{1, \dots, n_c \} $ and $ j \in \{ 1,\dots, n_{\epsilon}\}$, we take partial derivative of the $ i $-th dimension of $\mathbf{c}_{t}$ on both sides of Equation (\ref{eq:vt+1_esimated_to_true}) w.r.t. $ \epsilon_{t-1, j} $ and have:

\begin{equation}
\begin{split}
     0 =& \frac{ \partial f_{t+1, i} ( \mathbf{c}_{t},\mathbf{v}_{t}, \epsilon_{t+1} )}{ \partial \epsilon_{t-1,j} } 
     \\ =& \frac{ \partial \hat{f}_{t+1, i} (h_{\mathbf{c}_{t}, \mathbf{v}_{t},\epsilon_{t+1}} (\mathbf{c}_{t},\mathbf{v}_{t-2},\mathbf{v}_{t},\epsilon_{t-1},\epsilon_{t+1}) ) }{ \partial \epsilon_{t-1,j} } \\
     =&\sum_{k \in \{ 1, \dots, n_c+ n_v + n_{\epsilon} \} } \frac{ \partial \hat{f}_{t+1, i} }{ \partial h_{(\mathbf{c}_{t},\mathbf{v}_{t}, \epsilon_{t+1}), k} } \cdot \frac{ \partial h_{(\mathbf{c}_{t},\mathbf{v}_{t}, \epsilon_{t+1}), k} }{ \partial \epsilon_{t-1,j} }. 
\end{split}
\end{equation}

Since $\hat{f}_{t+1}$ is invertible, for $n_c+n_v+n_{\epsilon}$ different values of $\hat{f}_{t+1,i}$, each vector $ [ \frac{ \partial \hat{f}_{t+1, i} }{ \partial h_{(\mathbf{c}_{t}, \mathbf{v}_{t},\epsilon_{t+1}), 1} }, \dots, \frac{ \partial \hat{f}_{t+1, i} }{ \partial h_{(\mathbf{c}_{t},\mathbf{v}_{t}, \epsilon_{t+1}), n_c+n_v+n_{\epsilon}} }]$ are linearly independent. Therefore, the $(n_c+n_v+n_{\epsilon})\times(n_c+n_v+n_{\epsilon})$ linear system is invertible and has the unique solution as follows:
\begin{equation}
\label{eq:linear_dep_2}
    \frac{ \partial h_{(\mathbf{c}_{t},\mathbf{v}_{t}, \epsilon_{t+1}), k} }{ \partial \epsilon_{t-1,j} } ( \mathbf{c}_{t},\mathbf{v}_{t-2},\mathbf{v}_{t},\epsilon_{t-1},\epsilon_{t+1} ) = 0,
\end{equation}
meaning that $(\mathbf{c}_{t}, \mathbf{v}_{t},\epsilon_{t+1})$  does not depend on $ \epsilon_{t-1}$.

For $ i \in \{1, \dots, n_c \} $ and $ k \in \{ 1,\dots, n_v \}$, we take partial derivative of the $ i $-th dimension of $\mathbf{v}_{t}$ on both sides of Equation (\ref{eq:vt+1_esimated_to_true}) w.r.t. $ \mathbf{v}_{t-2, k} $ and have:
\begin{equation}
\begin{split}
     0 =& \frac{ \partial f_{t+1, i} ( \mathbf{c}_{t},\mathbf{v}_{t}, \epsilon_{t+1} )}{ \partial \mathbf{v}_{t-2,k} } 
     \\ =& \frac{ \partial \hat{f}_{t+1, i} (h_{\mathbf{c}_{t}, \mathbf{v}_{t},\epsilon_{t+1}} (\mathbf{c}_{t},\mathbf{v}_{t-2},\mathbf{v}_{t},\epsilon_{t-1},\epsilon_{t+1}) ) }{ \partial \mathbf{v}_{t-2,k} } \\
     =&\sum_{l \in \{ 1, \dots, n_c+ n_v + n_{\epsilon} \} } \frac{ \partial \hat{f}_{t+1, i} }{ \partial h_{(\mathbf{c}_{t},\mathbf{v}_{t}, \epsilon_{t+1}), l} } \cdot \frac{ \partial h_{(\mathbf{c}_{t},\mathbf{v}_{t}, \epsilon_{t+1}), l} }{ \partial \mathbf{v}_{t-2,k} },
\end{split}
\end{equation}
meaning that $(\mathbf{c}_{t}, \mathbf{v}_{t},\epsilon_{t+1})$  does not depend on $ \mathbf{v}_{t-2}$.
According to Equation (\ref{eq:h_func}), we have $(\hat{\mathbf{c}}_t,\hat{\mathbf{v}}_{t-2},\hat{\mathbf{v}}_{t}, \hat{\epsilon}_{t-1}, \hat{\epsilon}_{t+1})=h_\mathbf{c}(\mathbf{c}_t,\mathbf{v}_{t-2}, \mathbf{v}_{t}, \epsilon_{t-1}, \epsilon_{t+1})$. By using the fact that $(\mathbf{c}_{t}, \mathbf{v}_{t-2}, \epsilon_{t-1})$  does not depend on $ \epsilon_{t+1}$ and $\mathbf{v}_{t}$,  $(\mathbf{c}_{t}, \mathbf{v}_{t}, \epsilon_{t+1})$ does not depend on $ \epsilon_{t-1}$ and $\mathbf{v}_{t-2}$, we have $\hat{\mathbf{c}}_t=h_\mathbf{c}(\mathbf{c}_t)$, i.e., the latent action variables are block-wise identifiable. 
\end{proof}

% #############################
\section{Implementation Details}

\begin{figure}[t]
    \centering
    \includegraphics[width=0.95\linewidth]{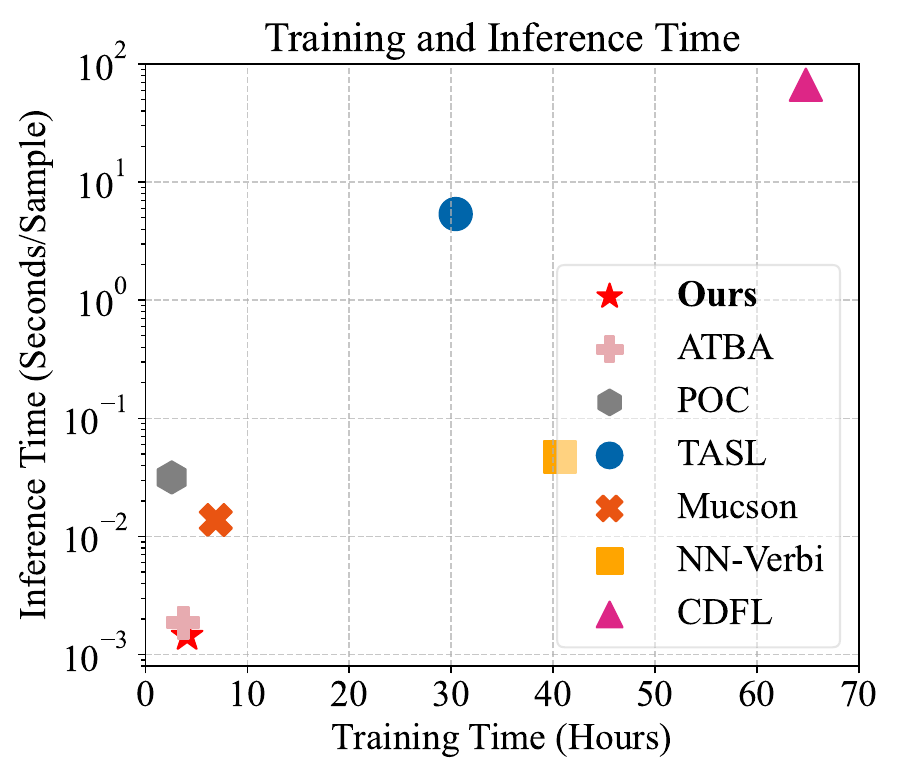}
    \caption{Comparison of training and inference time between the proposed framework and baseline methods. The framework universally outperforms baselines in terms of faster inference speed and shorter training time across all experimental settings.}
    \label{fig:train_inference_time}
\end{figure}

\begin{figure}[t]
    \centering
    \includegraphics[width=0.95\linewidth]{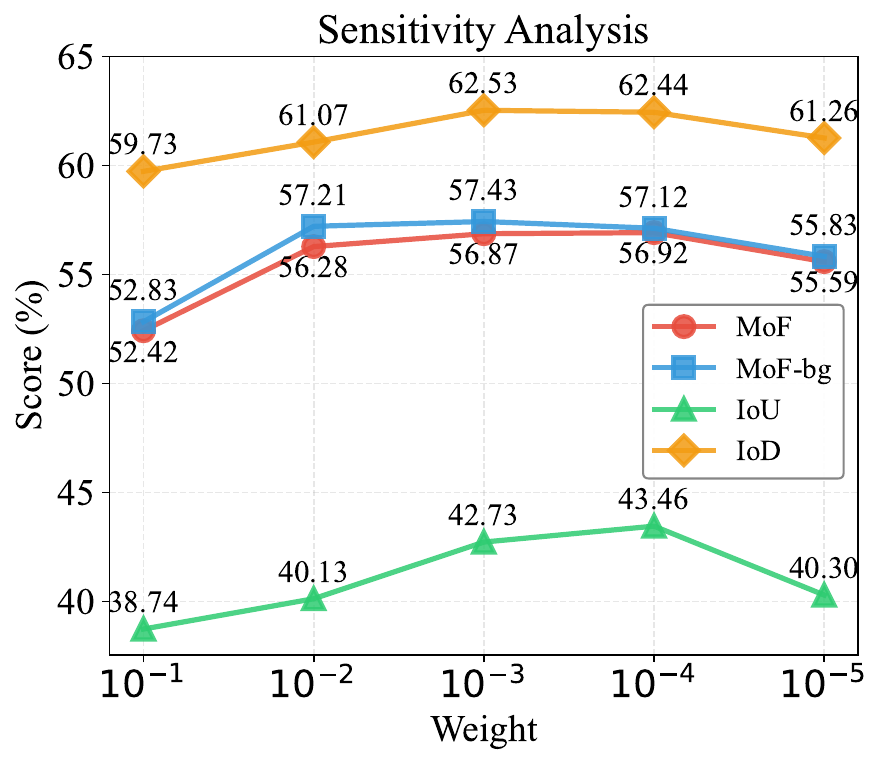}
    \caption{Sensitivity analysis of the hyperparameter for loss term $\mathcal{L}_s$ across datasets. Performance rises notably within a specific hyperparameter range and stabilizes thereafter, verifying the existence of a robust optimal hyperparameter range.}
    \label{fig:sensitivity_analysis}
\end{figure}

\label{app:imple}
\subsection{Model Details}
We choose ATBA\cite {ATBA} as the backbone and Transformer as the encoder and decoder, while the classifier is implemented using linear layers. Architecture details of the proposed method are shown in Table \ref{table:arch_details}.
\subsection{Experiment Details}
We use ADAMW optimizer in all experiments and report the Mean-over-Frames(MoF), Mean-over-Frames without Background(MoF-bg), Intersection-over-Union(IoU) and Intersection-over-Detection(IoD) as evaluation metrics. The experiments are implemented by Pytorch on a single NVIDIA RTX 3090 24GB GPU.

\begin{table}[t]
\caption{Architecture details. L:length of input feature, $|\mathbf{x}_t|$:the dimension of $\mathbf{x}_t$, $|y|$: the numbers of action class, ReLU:Rectified Linear Unit.}
\label{table:arch_details}
\resizebox{\columnwidth}{!}{%
% \begin{tabular}{@{}cHALcl@{}}
\begin{tabular}{@{}clclcl@{}}
\toprule
\multicolumn{2}{c}{Configuration}       & \multicolumn{2}{c}{Description}               & \multicolumn{2}{c}{Output}                                \\ \midrule
\multicolumn{2}{c|}{$1. \phi$}          & \multicolumn{2}{c|}{Transformer}                 & \multicolumn{2}{c}{}                                      \\ \midrule
\multicolumn{2}{c|}{Input $\mathbf{x}_{1:T}$} & \multicolumn{2}{c|}{Input Feature}            & \multicolumn{2}{c}{$Batch Size \times T \times |\mathbf{x}_t|$} \\
\multicolumn{2}{c|}{Dense}              & \multicolumn{2}{c|}{256 neurons}        & \multicolumn{2}{c}{$Batch Size \times T \times 256$}      \\ \midrule
\multicolumn{2}{c|}{$2. \psi$}          & \multicolumn{2}{c|}{Visual Encoder}           & \multicolumn{2}{c}{}                                      \\ \midrule
\multicolumn{2}{c|}{Input $\mathbf{z}_{1:T}$} & \multicolumn{2}{c|}{Feature Latent Variables} & \multicolumn{2}{c}{$Batch Size \times T \times 256$}      \\
\multicolumn{2}{c|}{Dense}              & \multicolumn{2}{c|}{256 neurons; ReLU}        & \multicolumn{2}{c}{$Batch Size \times T \times 256$}      \\ \midrule
\multicolumn{2}{c|}{$3. \eta$}          & \multicolumn{2}{c|}{Action Encoder}          & \multicolumn{2}{c}{}                                      \\ \midrule
\multicolumn{2}{c|}{Input $\mathbf{v}_{1:T}$} & \multicolumn{2}{c|}{Visual Latent Variables}  & \multicolumn{2}{c}{$Batch Size \times T \times 256$}      \\
\multicolumn{2}{c|}{Dense}              & \multicolumn{2}{c|}{256 neurons; ReLU}        & \multicolumn{2}{c}{$Batch Size \times T \times 256$}      \\ \midrule
\multicolumn{2}{c|}{$4. \kappa$}        & \multicolumn{2}{c|}{Visual Decoder}           & \multicolumn{2}{c}{}                                      \\ \midrule
\multicolumn{2}{c|}{Input $\mathbf{v}_{1:T}$} & \multicolumn{2}{c|}{Visual Latent Variables}  & \multicolumn{2}{c}{$Batch Size \times T \times 256$}      \\
\multicolumn{2}{c|}{Dense}              & \multicolumn{2}{c|}{512 neurons; ReLU}        & \multicolumn{2}{c}{$Batch Size \times T \times 512$}      \\ \midrule
\multicolumn{2}{c|}{$5. \xi$}       & \multicolumn{2}{c|}{Action Decoder}          & \multicolumn{2}{c}{}                                      \\ \midrule
\multicolumn{2}{c|}{Input $\mathbf{c}_{1:T}$} & \multicolumn{2}{c|}{Visual Latent Variables}  & \multicolumn{2}{c}{$Batch Size \times T \times 256$}      \\
\multicolumn{2}{c|}{Dense}              & \multicolumn{2}{c|}{512 neurons; ReLU}        & \multicolumn{2}{c}{$Batch Size \times T \times 512$}      \\ \midrule
\multicolumn{2}{c|}{$6. \Gamma$}        & \multicolumn{2}{c|}{Classifier}               & \multicolumn{2}{c}{}                                      \\ \midrule
\multicolumn{2}{c|}{Input $\mathbf{c}_{1:T}$} & \multicolumn{2}{c|}{Action Latent Variables} & \multicolumn{2}{c}{$Batch Size \times T \times 256$}      \\
\multicolumn{2}{c|}{Dense}              & \multicolumn{2}{c|}{$|y|$ neurons; ReLU}      & \multicolumn{2}{c}{$Batch Size \times T \times |y|$}      \\ \bottomrule
\end{tabular}%
}
\end{table}

\subsection{Basic Training Settings}
\textcolor{black}{All model training adopts the AdamW optimizer with hyperparameter configurations set as 1e-4 for weight decay and 5e-4 for initial learning rate; the model is trained for 400/300/200/300 epochs for Breakfast~\cite{kuehne2014language}, Hollywood~\cite{bojanowski2014weakly}, Crosstask~\cite{zhukov2019cross}, GTEA~\cite{gtea}, respectively.}

\subsection{Computational Efficiency Analysis}
\textcolor{black}{This section compares the computational efficiency of the proposed framework against baseline methods, focusing on training and inference time. As shown in Fig.~\ref{fig:train_inference_time}, the framework achieves faster inference speed and shorter training time compared to existing baselines across all experimental settings.}

\subsection{Hyperparameter Sensitivity Analysis}
\textcolor{black}{This section analyzes the sensitivity of the hyperparameter for loss term $\mathcal{L}_s$ across different datasets. As shown in Fig.~\ref{fig:sensitivity_analysis}, the performance of the proposed framework improves significantly as the hyperparameter value increases within a specific range, and then tends to stabilize with further increases in the hyperparameter value.}

% #############################
\section{Additional Experimental Results}
\label{app:experiments}
\subsection{Ablation Study Extensions}
\textcolor{black}{The minor performance gaps across distinct loss function configurations reflect the inherent design principle of our framework: rather than relying on a single loss term to drive performance gains, the model leverages the complementary nature of multiple loss objectives to achieve incremental improvements. An analysis of the extended ablation study results (Table~\ref{tab:more_ablation}) reveals that no individual loss term or simple pairwise combination (e.g., $\mathcal{L}_r + \mathcal{L}_{KL}$) can match the performance of multi-loss combinations, which demonstrates that each loss term targets a unique aspect of the learning process—such as reconstruction fidelity or latent space regularization—and their integration addresses the limitations of single-objective optimization.}

\begin{table}[]
\caption{Experiment results on the GTEA dataset.}
\vspace{-1em}
\label{tab:more_ablation}
\resizebox{\columnwidth}{!}{%
\begin{tabular}{@{}c|ccccccc@{}}
                \toprule
                EXP & $\mathcal{L}_r$ & $\mathcal{L}_s$ & $\mathcal{L}_{KL}$ & $\delta$ & MoF & IoU & IoD \\ \midrule
                1 & & & & & 53.3 & 40.9 & 58.7 \\
                2 & \ding{51} & & & & 54.3 & 38.4 & 61.6 \\
                3 & & \ding{51} & & & 54.6 & 40.3 & 61.6 \\
                4 & & & \ding{51} & & 54.5 & 42.0 & 61.0 \\
                5 & & & & \ding{51} & 53.9 & 41.1 & 59.4 \\
                6 & & \ding{51} & \ding{51} & & 54.7 & 39.3 & 61.4 \\
                7 & \ding{51} & & \ding{51} & & 55.2 & 42.4 & 60.7 \\
                8 & \ding{51} & \ding{51} & & & 55.4 & 41.0 & 60.9 \\
                9 & & \ding{51} & \ding{51} & \ding{51} & 56.4 & 41.9 & 61.6 \\
                10 & \ding{51} & \ding{51} & & \ding{51} & 55.5 & 40.8 & 60.9 \\
                11 & \ding{51} & \ding{51} & \ding{51} & & 56.4 & 41.6 & \textbf{63.3} \\ \midrule
                12 & \ding{51} & \ding{51} & \ding{51} & \ding{51} & \textbf{56.6} & \textbf{42.6} & 62.1 \\ \bottomrule
            \end{tabular}
}
\end{table}

\subsection{Linear Probing Results}
\textcolor{black}{Linear probing results (Fig.~\ref{fig:linear_probing}) demonstrate that action latent variables yield higher F1 scores than raw visual features, with this performance divergence rooted in the distinct semantic encoding characteristics of the two feature forms. Visual features tend to capture transient, task-irrelevant visual fluctuations, whereas action latent variables, regularized by the hierarchical causal generative process, encode stable, temporally consistent semantic information of action states. The superior linear probing performance of action latent variables validates that the learned representations capture action-specific semantic structures instead of superficial visual patterns, which confirms the effective disentanglement of hierarchical latent variables and the efficacy of temporal smoothness regularization on the hierarchical causal data generation process.}
\begin{figure}
    \centering
    \includegraphics[width=0.95\linewidth]{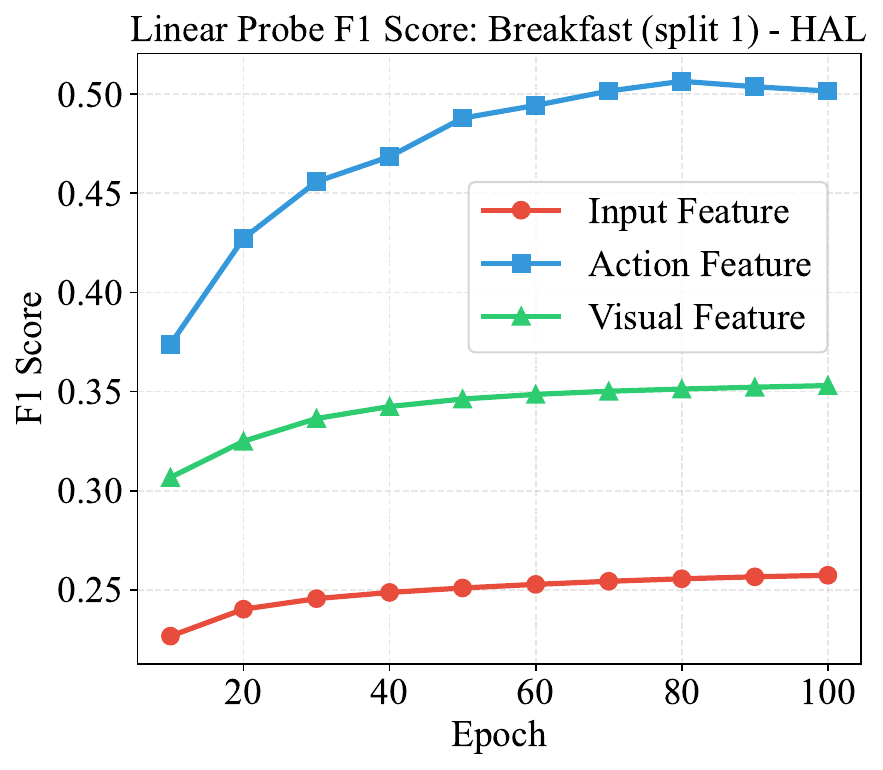}
    \caption{The F1 score with linear probing on Breakfast.}
    \label{fig:linear_probing}
\end{figure}

\subsection{Empirical Motivation}
\textcolor{black}{This section empirically validates the core motivation of the proposed method, as shown in Fig.~\ref{fig:empirical_motivation}. Frame-to-frame visual feature distances exhibit large and frequent fluctuations, while ground-truth action boundaries typically occur after prolonged visual fluctuations instead of coinciding with such variations.}
\begin{figure}[htbp]
    \centering
    \includegraphics[width=0.9\linewidth]{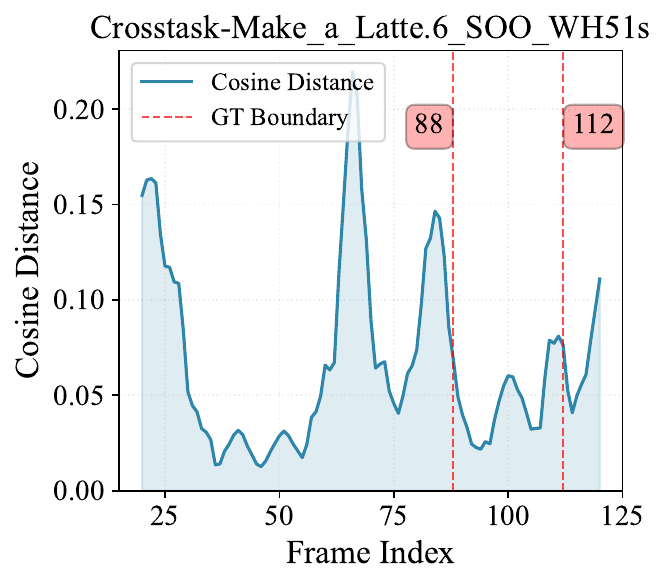}
    \caption{Frame-to-frame visual feature distance and ground-truth action boundary distribution. Visual features show large frequent fluctuations, and action boundaries arise after prolonged visual variations rather than coinciding with them.}
    \label{fig:empirical_motivation}
\end{figure}

% #############################
\section{Discussion of the Identification Results}
\subsection{Bounded, Continues and Positive Density}
\label{app:discussion}
The assumptions used in our theoretical analysis are also common in the field of identification, such as \cite{hu2008instrumental} and \cite{kong2023understanding}. For example, the bounded and continuous density assumption requires that both observed and latent variables are continuous and lie within bounded ranges. In video data, such conditions are typically satisfied because pixel values and latent transitions change smoothly over time and usually remain within reasonable bounds, which ensures the corresponding probability densities are continuous and bounded. The positive density assumption is also achievable with sufficiently rich observational data. Moreover, injective linear operators imply that the information in the latent variables is equivalent to that contained in continuous multi-frame video observations. Please refer to Appendix B for the detailed explanations of the assumptions, how they relate to real-world scenarios, and under which conditions they are satisfied.

\subsection{Injective Linear Operators}
A common practice in nonparametric identification involves leveraging the injectivity property of linear operators \cite{hu2008instrumental,carroll2010identification}. At its core, this idea signifies that distinct input distributions fed into a linear operator will produce distinct output distributions from that same operator. To foster a clearer grasp of this premise, we present a series of illustrations depicting the distributional mapping \(p_{\mathbf{a}} \rightarrow p_{\mathbf{b}}\), where \(\mathbf{a}\) and \(\mathbf{b}\) represent random variables.
\begin{example}[\textbf{Inverse Transformation}]
\label{equ:example_inverse}
    $b = g(a)$, where $g$ is an invertible function.
\end{example}

\begin{example}[\textbf{Additive Transformation}]
\label{equ:additive_transfromation}
    $b = a + \epsilon$, where $p(\epsilon)$ must not vanish everywhere after the Fourier transform (Theorem 2.1 in \cite{mattner1993some}).
\end{example}

\begin{example}
\label{equ:additive_transfromation2}
    $b = g(a) + \epsilon$, where the same conditions from Examples \ref{equ:example_inverse} and \ref{equ:additive_transfromation} are required.
\end{example}

\begin{example}
    [\textbf{Post-linear Transformation}] $b = g_1(g_2(a) + \epsilon)$, a post-nonlinear model with invertible nonlinear functions $g_1, g_2$, combining the assumptions in \textbf{Examples \ref{equ:example_inverse}-\ref{equ:additive_transfromation2}}.
\end{example}

\begin{example}
    [\textbf{Nonlinear Transformation with Exponential Family}]
    $b = g(a, \epsilon)$, where the joint distribution $p(a, b)$ follows an exponential family.
\end{example}

\begin{example}
[\textbf{General Nonlinear Transformation}]
    $b = g(a, \epsilon)$, a general nonlinear formulation. Certain deviations from the nonlinear additive model (\textbf{Example} \ref{equ:additive_transfromation2}), e.g., polynomial perturbations, can still be tractable.
\end{example}

\subsection{Non-singular Jacobia}
This assumption also finds frequent application in the work of \cite{kong2023understanding}. From a mathematical perspective, it indicates that the Jacobian matrix mapping latent variables to observed variables is of full rank. In practical situations, this translates to each latent variable having a corresponding observation. To meet this requirement, those independent latent variables that exert no impact on observations can be readily excluded.

% #############################

\section{More Qualitative Results}
We provide more qualitative results on Breakfast~\cite{kuehne2014language}, Hollywood~\cite{bojanowski2014weakly} and GTEA~\cite{gtea} datasets in figure \ref{fig:more_qualit}.

\begin{figure}[!ht]  % [!ht] 
\centering
\begin{subfigure}[b]{\linewidth}  
  \centering
  \includegraphics[width=0.95\linewidth]{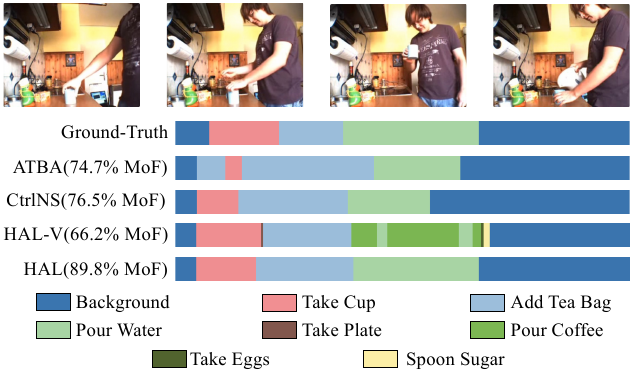}  
  \caption{\textit{P07-stereo01-P07-tea} on Breakfast.}
  \label{fig:qua_P07}
\end{subfigure}
% \vspace{1pt}

\begin{subfigure}[b]{\linewidth}  
  \centering
  \includegraphics[width=0.95\linewidth]{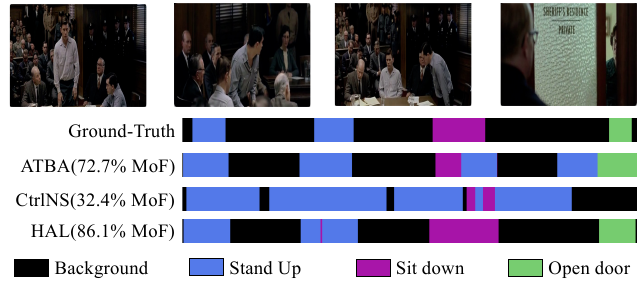}  
  \caption{\textit{0146} on Hollywood.}
  \label{fig:qua_0146}
\end{subfigure}
% \vspace{1pt}

\begin{subfigure}[b]{\linewidth}
  \centering
  \includegraphics[width=0.95\linewidth]{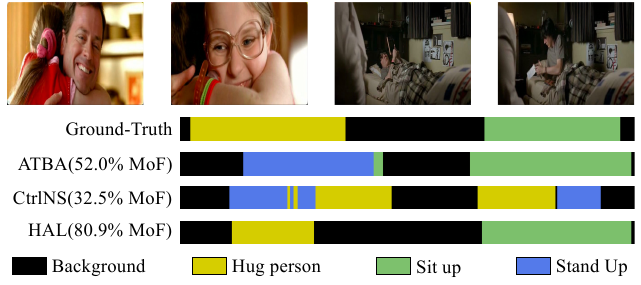}
  \caption{\textit{0782} on Hollywood.}
  \label{fig:qua_0782}
\end{subfigure}
% \vspace{1pt}

\begin{subfigure}[b]{\linewidth}  
  \centering
  \includegraphics[width=0.95\linewidth]{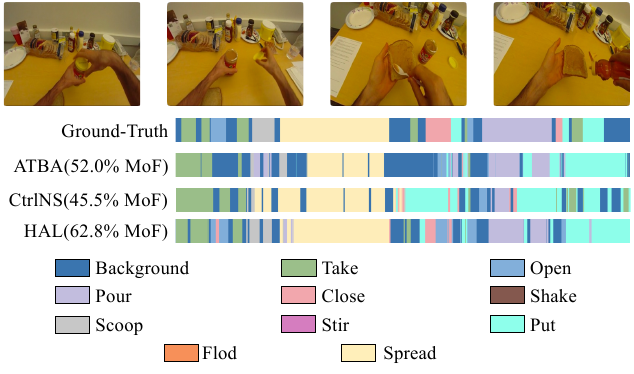}  
  \caption{\textit{S1-Peanut-C1} on GTEA.}
  \label{fig:qua_peanut}
\end{subfigure}
% \vspace{1pt}

\begin{subfigure}[b]{\linewidth} 
  \centering
  \includegraphics[width=0.95\linewidth]{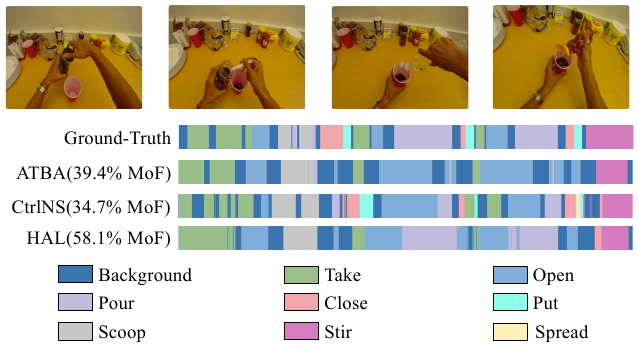}  
  \caption{\textit{S4-CofHoney-C1} on GTEA.}
  \label{fig:qua_cofhoney}
\end{subfigure}

\caption{More qualitative results.}
\label{fig:more_qualit}
\end{figure}

\end{document}